\documentclass{article}

\PassOptionsToPackage{numbers, sort&compress}{natbib}

\usepackage[main, final]{neurips_2026}

\usepackage[utf8]{inputenc} 
\usepackage[T1]{fontenc}    
\usepackage[dvipsnames]{xcolor}
\usepackage{hyperref}       
\hypersetup{
    colorlinks=true,
    citecolor=PineGreen  
}
\usepackage{url}            
\usepackage{booktabs}       
\usepackage{amsfonts}       
\usepackage{nicefrac}       
\usepackage{microtype}      
\usepackage{xcolor}         
\usepackage{subcaption}
\usepackage{wrapfig}
\usepackage{enumitem}

\usepackage{amsmath}
\usepackage{amssymb}
\usepackage{mathtools}
\usepackage{amsthm}
\theoremstyle{plain}
\newtheorem{theorem}{Theorem}[section]
\newtheorem{proposition}[theorem]{Proposition}

\newtheorem{corollary}[theorem]{Corollary}
\theoremstyle{definition}
\newtheorem{definition}[theorem]{Definition}
\newtheorem{assumption}[theorem]{Assumption}
\theoremstyle{remark}

\usepackage{amsmath,amsfonts,bm}
\usepackage{upgreek}

\def\eqref#1{equation~\ref{#1}}

\def\1{\bm{1}}

\def\rrho{{\uprho}}

\def\rf{{\textnormal{f}}}
\def\rg{{\textnormal{g}}}
\def\rh{{\textnormal{h}}}

\def\rx{{\textnormal{x}}}
\def\ry{{\textnormal{y}}}
\def\rz{{\textnormal{z}}}

\def\rvc{{\mathbf{c}}}

\def\rvf{{\mathbf{f}}}
\def\rvg{{\mathbf{g}}}

\def\rvx{{\mathbf{x}}}

\def\rvz{{\mathbf{z}}}

\def\vf{{\bm{f}}}
\def\vg{{\bm{g}}}

\def\vv{{\bm{v}}}

\def\vx{{\bm{x}}}

\def\mC{{\bm{C}}}

\DeclareMathAlphabet{\mathsfit}{\encodingdefault}{\sfdefault}{m}{sl}
\SetMathAlphabet{\mathsfit}{bold}{\encodingdefault}{\sfdefault}{bx}{n}

\DeclareMathOperator*{\argmax}{arg\,max}

\title{Human-AI Teaming Through the Lens of Calibration}

\author{
  \parbox[t]{0.45\linewidth}{\centering
    \textbf{Eric Nalisnick}\\
    \normalfont
    Department of Computer Science\\
    Johns Hopkins University\\
    \texttt{nalisnick@jhu.edu}
  }
  \And
  \parbox[t]{0.45\linewidth}{\centering
    \textbf{Chi Zhang}\thanks{Equal contribution.}\\
    \normalfont
    Department of Computer Science\\
    Johns Hopkins University\\
    \texttt{czhan168@jhu.edu}
  }
  \AND
  \parbox[t]{0.45\linewidth}{\centering
    \textbf{Sophia Qian}$^{*}$\\
    \normalfont
    Department of Computer Science\\
    Johns Hopkins University\\
    \texttt{cqian17@jh.edu}
  }
  \And
  \parbox[t]{0.45\linewidth}{\centering
    \textbf{Yixin Wang}\\
    \normalfont
    Department of Statistics\\
    University of Michigan\\
    \texttt{yixinw@umich.edu}
  }
}

\begin{document}

\maketitle

\begin{abstract}\looseness=-1
  We study models for human-AI teaming through the lens of statistical calibration.  We assume the team consists of an AI model and human---both of which are calibrated with respect to some partitioning of the feature space---and expose how the calibration assumptions propagate into the teaming framework.  In particular, we consider frameworks that either (i) combine human and model predictions or (ii) delegate prediction responsibility to either a human or model.  We show via theoretical and empirical results that existing methods for combination do not preserve the human's degree of calibration.  Methods for delegation (by the very act of delegation) preserve calibration of the downstream predictors but shift the burden onto the rejector meta-model that decides who predicts. The rejector must be calibrated finely enough to locate where each member is superior, a demand that grows with the human's expertise and becomes unattainable when the human relies on information the system cannot observe.
\end{abstract}

\looseness=-1
\section{Introduction}
\label{sec:intro}
Modern AI systems can now extract insights from large scale data sources in ways that far surpass human abilities.  However, humans still can contribute meaningful contextual information \citep{alur2024human, li2026scaling, recht2025actuary}.  For instance, a medical AI system can inform its diagnosis from millions of scientific publications and patient health records.  Yet this AI system cannot know a particular patient as well as a doctor who has been caring for this patient for years.  This complementarity in information---population-level trends vs nuanced, specific insights---motivates the use of \textit{hybrid intelligent} systems \citep{grosz1996collaborative, kamar2016directions}.  Ideally, these systems leverage human and machine decision-making such that the human-AI team's performance is better than either could achieve alone.  This essential property is called \textit{complementarity} \citep{trautman2017mathematical, donahue2022human, SteyversComplement}.

Human-AI teams have been observed to have complex dynamics, and a full characterization of when teaming succeeds and fails is still an open problem \citep{steyvers2024three}.  For example, \citet{Bansal2019Updates} reported a user study in which improving the AI model actually decreased team performance.  Thus one would hope for the team to exhibit certain robustness properties \citep{peng2025no}: for instance, the team's predictions should never be worse than its best member's.  Or in frameworks that delegate responsibility \citep{madras2018predict}, we would never want the act of delegation to depend on having a full characterization of the behavior of all team members.  If it did, then we could just use the delegation system to make predictions and disband the team.  While work has attempted to address these questions empirically \citep{towardsAScience} or theoretically in specific settings \citep{trautman2017mathematical, mozannar2023exact, donahue2022human, collina2026collaborative, guo2024decision}, no prior work has addressed these questions across multiple teaming settings from a unified theoretical perspective.  

In this work, we use statistical calibration as a unified theoretical lens through which to study human-AI teaming.  We consider two teaming frameworks: combining human and model predictions \citep{kerrigan2021} and delegating prediction tasks to one team member \citep{madras2018predict}.  We assume that the team is composed of one model and one human---both of which are calibrated for some partitioning of the feature space.  Calibration is a powerful tool for studying teaming as it allows for a general representation of the human and model's prediction quality.  As the partitioning becomes increasingly fine-grained, the human / model approaches the Bayes predictor.  As the partitioning becomes increasingly coarse, the human / model approaches marginal calibration (i.e.~uniform confidence of $1/K$ for balanced $K$-class classification).  Through a series of theoretical results, simulations, and experiments with authentically human-made predictions, we expose how the calibration assumptions propagate into the team's predictive behavior.  In particular, we arrive at mostly \emph{negative} results that show common teaming procedures do not preserve the underlying calibration of the team members.  To summarize our contributions, we:\looseness=-1\begin{itemize}[leftmargin=10pt]
    \item Prove that model-based combination preserves calibration w.r.t.~the model (Proposition \ref{prop:cal_for_f}) but not  w.r.t.~the human (Theorem \ref{thm:negative}).  We verify this result via simulation (Figure \ref{fig:model_comb_sim}).
    \item Prove that the combination approach of \citet{kerrigan2021} preserves calibration w.r.t.~the model (Theorem \ref{thm_cal_bayes_f}) but not  w.r.t.~the human (Proposition \ref{thm:neg_bayes_h}).  Yet this preservation is only under independent partitioning, and as the model approaches the Bayes classifier, independence is violated and the team predictor becomes increasingly miscalibrated (Proposition \ref{prop_post_miscal_f_bayes}).  Experiments on human predictions via ImageNet-16H  agree with our theoretical finding (Figure \ref{fig:empirical_validation}).
    \item Prove that delegation-based teaming \citep{pmlr-v119-mozannar20b} requires the meta-classifier (a.k.a.~rejector) be sufficiently calibrated as to characterize where the human and model are the superior predictors (Theorem \ref{thm:cal_rejector}).  When the human has access to hidden features, the rejector will likely have irreducible excess risk.
\end{itemize}

\looseness=-1
\section{Background \& Related Work}\looseness=-1
\paragraph{Notation \& Setting} We consider a $K$-class classification task with labels $\ry \in \mathcal{Y} = \{1, \ldots, K\}$, \emph{known} features $\rvx \in \mathfrak{X}$.  We will consider two settings: shared and independent feature spaces.  For the former, the labels are drawn from an unknown distribution $\mathbb{P}(\ry | \rvx)$, and for the latter, the labels are drawn from $\mathbb{P}(\ry | \rvx, \rvz)$, with $\rvz \in \mathcal{Z}$ being features that are known only to the human.  The human makes predictions $\rh \in \mathcal{H} = \mathcal{Y}$ by sampling from an internal predictive model of the form $\mathbb{P}(\rh | \rvx)$, $\mathbb{P}(\rh | \rvz)$, or $\mathbb{P}(\rh | \rvx, \rvz)$---depending on if the human has access to the same ($\rvx$), a different ($\rvz$), or a strictly richer ($\rvx, \rvz$) feature space than the model.  We observe the known features and labels only through i.i.d.~samples $(\vx , y)$, and thus when $\rvz$ contains essential predictive information, a model trained only on $(\vx , y)$ cannot recover the ground-truth distribution $\mathbb{P}(\ry | \rvx, \rvz)$.  Predictive models are considered to be a map $\rvf: \mathfrak{X} \mapsto \Delta^{K-1}$, where $\Delta^{K-1}$ is the (K-1)-dimensional simplex. The model's distribution over class labels is then $\ry \sim \text{\texttt{Categorical}}(\rvf(\rvx))$, where $\rf_{y}(\rvx)$ denotes the classifier's confidence assigned to label $y$ and is usually parameterized with the softmax function.   

\looseness=-1
\subsection{Statistical Calibration}\label{sec:cal_intro}
We now introduce statistical calibration \citep{murphy1967verification, dawid1982well} for an abstract classifier $\rvf:\mathfrak{X} \mapsto \Delta^{K-1}$. Let $\Phi: \mathfrak{X} \mapsto \mathcal{U}$ be a \textit{grouping function} that partitions the feature space $\mathfrak{X}$ into reference classes by assigning each point to a group indexed by $u \in \mathcal{U}$.
\begin{definition}\label{def:cal}  \ \textbf{Pointwise Calibration.}
A predictor $\rvf$ is pointwise calibrated w.r.t.~a choice of $\Phi$ if \begin{equation*}   f_{y}(\vx) \ = \ \mathbb{E}\left[ \mathbb{P}(\ry=y | \rvx) \ | \ \rvx \in [\vx]_{\Phi} \right] \ = \ \mathbb{P}\left(\ry=y | \Phi(\vx)\right), \quad \forall \vx \in \mathfrak{X}, \ \forall y \in \mathcal{Y},
    \end{equation*}\end{definition} where $[\vx]_{\Phi} = \left\{\vx' | \Phi(\vx) = \Phi(\vx')  \right\} \subseteq \mathfrak{X}$ denotes the equivalence class of $\rvx$.  Therefore $\rvf$ is calibrated if the confidence it assigns to every $(\vx, y)$ pair equals the \textit{average} probability of that label within the group to which $\vx$ is assigned.  Appropriately choosing $\Phi$ is known as the \textit{reference class problem} \citep{hajek2007reference} as the choice of $\Phi$ naturally gives rise to stronger and weaker forms of calibration \citep{pmlr-v89-vaicenavicius19a, 10.1145/3593013.3594068}.  The strongest form is \textit{ideal calibration}, which means the predictor is the Bayes predictor, having matched all underlying conditional probabilities: $f_{y}(\vx) \ = \ \mathbb{P}\left(\ry=y | \vx \right).$ Ideal calibration is `strong' because the grouping function is `fine grained,' having $\Phi(\rvx) = \rvx$.   On the other hand, a model is weakly calibrated if calibration only holds for a `coarse grained' grouping function.  The weakest form is marginal calibration, which occurs when $\Phi(\rvx)$ partitions the feature space into a single group: $f_{y}(\vx) = \mathbb{P}(\ry = y)$.  In the machine learning literature, \textit{confidence calibration} is typically of interest, which uses the maximum class confidence as the grouping function: $\Phi(\rvx) = \max_{y} f_{y}(\rvx)$ \citep{guo2017calibration}.  If a predictor is calibrated for multiple grouping functions, then it is called \textit{multi-calibrated}  \citep{pmlr-v80-hebert-johnson18a}.

Definition \ref{def:cal} presumes that the predictor is itself measurable w.r.t.~$\Phi$ (i.e.~$\rvf$ is a function of $\Phi(\rvx)$).  Yet we will also want to know if a predictor  that is built from features that are not a function of $\Phi$ is still calibrated w.r.t.~$\Phi$.  For such predictors, we can use the following conditional notion of calibration, which coincides with Definition \ref{def:cal} whenever the predictor is $\Phi$-measurable.
\begin{definition}  \ \textbf{Generalized Calibration.}\label{def:comp_cal}
Let $\rvg$ be a predictor that is a measurable function of inputs generating a $\sigma$-algebra $\mathcal{S}$ (so that $\sigma(\rvg) \subseteq \mathcal{S}$), and let $\Phi$ be a grouping function.  We say $\rvg$ is calibrated w.r.t.~$\Phi$ if conditioning on $\Phi$ leaves the prediction unchanged:
\begin{equation*}
\rg_{y} \ = \ \mathbb{P}\!\left(\ry = y \mid \rvg = \vg, \Phi(\rvx) = u \right) \ = \ \mathbb{P}\!\left(\ry = y \mid \rvg = \vg \right), \quad \forall y \in \mathcal{Y}, \ \  \forall (\vg, u) \in \operatorname{supp}(\rvg, \Phi(\rvx)).
\end{equation*} \end{definition} This statement means that $\Phi$ carries no label information beyond what $\rvg$ already encodes.  For predictors that are not $\Phi$-measurable, this definition is strictly weaker than pointwise calibration.  It does \emph{not} require $\rvg$ to be constant within a cell of $\Phi$, as $\rvg$ may vary according to information that $\Phi$ does not encode.  When establishing a negative result, showing the failure of generalized calibration implies the failure of pointwise calibration.

\looseness=-1
\subsection{Frameworks for Human-AI Teaming}  While there are many ways to form human-AI teams, there are three fundamental workflows from which the vast majority of approaches can be derived \citep{10.1145/3802522}.  See Figure \ref{fig:haic_workflow} for a diagram of each.  The first is (human) prediction with AI assistance, and this is by far the most popular workflow \citep{kleinbergDecisions, zhang2020effect, bansal2021most, wholeExceed, towardsAScience}.  For example, a radiologist using computer vision to inform their analysis of a medical image would fit into this paradigm \citep{yu2024heterogeneity}.  Receiving help via discussions with a chatbot would be another example.  The second is less commonly found in the literature but rather obvious: combine the machine and human predictions using some voting or other ensembling rule (or even an additional model) \citep{CLEMEN1989559, kamarCrowds, kerrigan2021, pmlr-v238-showalter24a, pmlr-v267-kelly25a}.  This workflow was exploited by \citet{astrocombo} to identify images of supernovae by combining classifications from volunteers and a ConvNet.  The last is also popular: delegate prediction-making to either the human or machine \citep{madras2018predict, pmlr-v119-mozannar20b, Wilder2020LearningTC, okati2021differentiable}.  This workflow allows for semi-automation as hopefully the model can make the majority of predictions, and the human only needs to be queried sparingly.  This framework is now commonly employed for online content moderation: content that is obviously prohibited is usually detected by a model, and the ambiguous cases are deferred to a human, who can leverage cultural context and common sense to make the final determination \citep{10.1145/3491102.3501999}.

\begin{figure}
    \centering
	\begin{subfigure}[b]{.33\textwidth}
		\centering
        \includegraphics[width=.8\textwidth]{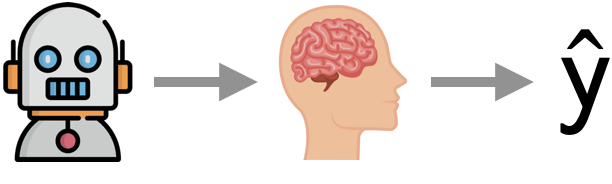} \\
        \subcaption{Prediction with AI Assistance}
        \label{fig:haic_workflow1}
	\end{subfigure}
	\hfill
    \begin{subfigure}[b]{.3\textwidth}
		\centering
        \includegraphics[width=0.5668\textwidth]{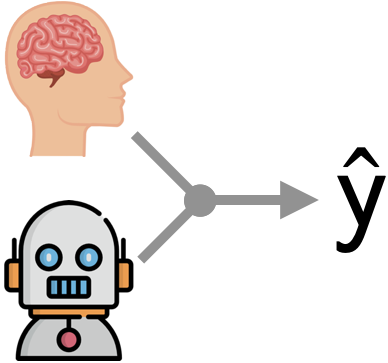} \\
        \subcaption{Combining Predictions}
        \label{fig:haic_workflow2}
	\end{subfigure}
	\hfill
    \begin{subfigure}[b]{.3\textwidth}
		\centering
        \includegraphics[width=.76\textwidth]{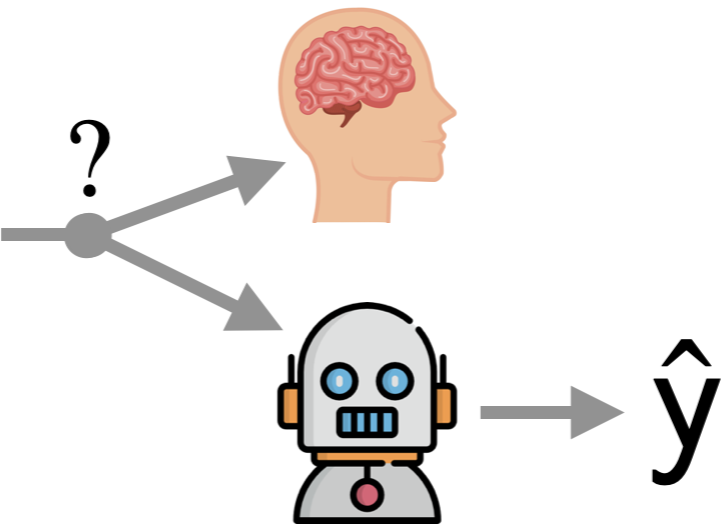} \\
        \subcaption{Delegating Predictions}
        \label{fig:haic_workflow4}
	\end{subfigure}
    \caption{\textit{Workflows for Human-AI Teaming}. $\hat{\ry}$ denotes the prediction.  Subfigure (a) shows \textit{AI assistance}: the human's decision is informed by the model's output.  Subfigure (b) shows \textit{combination}: the human and model's predictions are combined using some ensembling procedure.  Subfigure (c) shows \textit{delegation}: each prediction is delegated to the human or model.}
    \label{fig:haic_workflow}
\end{figure}

\looseness=-1
\subsubsection{AI Assistance and Prior Work on Calibration} Returning to the assistance framework, this formulation of teaming has already been well-studied from the perspective of calibration.  Several user studies have been performed during which users were shown confidence scores along with the AI's predictions \citep{zhang2020effect, wholeExceed, towardsAScience, 10.1145/3706598.3713336}, with the belief that confidence scores that are better calibrated should improve the quality of assistance \citep{tejeda2023displaying}.  However both \citet{vodrahalli2022uncalibrated} and \citet{cao2024designing} observed that this is not the case in practice.  \citet{vodrahalli2022uncalibrated} found that humans tended to perform better when confidences are skewed toward under/over confidence.  \citet{cao2024designing} observed benefits from using confidence scores but these benefits were agnostic to if the score was calibrated.  \citet{Benz2023Aligned} provide a theoretical framing of this phenomenon that explains why traditional confidence calibration did not improve the team's performance.  If the human has access to additional information ($\rvz$, in our notation), then calibrating the assistance model only with respect to $\Phi(\mathfrak{X})$ is insufficient.  Rather, a model with good assistance properties should be multi-calibrated w.r.t.~the human's own confidence scores, thereby accounting for the hidden features.  This work is our primary inspiration for examining other teaming workflows through the lens of calibration and shared vs independent feature spaces.  Lastly, \citet{peng2025no} characterized a no-free-lunch result for AI assistance, showing that when given calibrated confidences from two or more predictors, there does \emph{not} exist a (non-trivial) deterministic aggregation mechanism that can never perform worse than the worse predictor.

\looseness=-1
\subsubsection{Combining AI and Human Predictions}\label{subsec:combo}
We next turn to frameworks that combine human and model predictions.  These will be the first target of our analysis.  We will examine two specific combination methods from the literature: model-based combination \citep{astrocombo} and Bayesian combination of classifiers and confusion matrices \citep{kerrigan2021}.

\looseness=-1
\paragraph{Model-Based Combination}  We start with the obvious approach to combination.  Given a dataset $\mathcal{D} = \{\vx_n , y_n, h_n\}_{n=1}^{N}$, first fit a model $\vf(\rvx)$ to $\rvx$-$\ry$ pairs and then combine its predictions with a human's via a second model.  We define this second model as $g: \Delta^{K-1} \times \mathcal{H} \mapsto \Delta^{K-1}$.  \citet{astrocombo} used a simple combination rule and a support vector machine to implement $g(\rvf(\rvx), \rh)$.  This setting is attractive for its simplicity and will serve as a useful `warm up' for our line of analysis. 

\looseness=-1
\paragraph{Bayesian Combination of Classifiers and Confusion Matrices} \citet{kerrigan2021} proposed a more elegant approach to combination.  Assume access to two datasets: one of feature-label pairs $\mathcal{D}_{f} = \{(\vx_n , y_n)\}_{n=1}^{N}$ and another of prediction-label pairs $\mathcal{D}_{h} = \{(h_m , y_m)\}_{m=1}^{M}$.  They first fit a classifier $f(\rvx)$ to $\mathcal{D}_{f}$.  Then they use $\mathcal{D}_{h}$ to estimate a row-normalized confusion matrix $\mC$: $c_{i,j} = \sum_m (\mathbb{I}[y_m = i] \cdot \mathbb{I}[h_m = j]) / \sum_m \mathbb{I}[y_m = i] $. \citet{kerrigan2021} lastly combine the two models via Bayes rule as follows: \begin{equation}\begin{split}
    p(\ry | \rh, f(\rvx)) \ &= \ \frac{p(\rh | \ry, f(\rvx)) \cdot p(\ry | f(\rvx))}{p(\rh | f(\rvx))} \\
   \text{\textcolor{gray}{(assuming $\ f(\rvx) \ \bot \ \rh \ | \ \ry$)}} \ \ \ &= \ \frac{p(\rh | \ry) \cdot p(\ry | f(\rvx))}{\sum_{y' \in \mathcal{Y}} p(\rh | \ry=y') \cdot p(\ry=y' | f(\rvx))} \ = \ \frac{c_{\ry, \rh} \cdot f_{\ry
    }(\rvx)}{\sum_{y' \in \mathcal{Y}}c_{y', \rh} \cdot f_{y'
    }(\rvx) }
\end{split}
\end{equation} where the second equality assumes that the human's predictions are independent of $f(\rvx)$.  This assumption is then what allows for a standard confusion matrix to encode $p(\rh | \ry)$. 


\looseness=-1
\subsubsection{Delegating Predictions: Learning to Defer}\label{sec:l2d_bg}
The third workflow for teaming consists of delegating each prediction task to either the human or model via a routing policy \citep{holmstrom1980theory, raghu2019algorithmic, 10687308, de2020regression, de2021classification, fuchs2024optimizing}.  This paradigm is best known as \textit{learning to defer} (L2D) \citep{madras2018predict}, as the system must learn when to use which downstream predictor.  This allows for semi-automation, with the human only needing to make a (usually small) fraction of the decisions.  We can derive the Bayes optimal deferral rule as follows \citep{pmlr-v119-mozannar20b}. Consider the following cost function for a decision rule $\delta:\mathfrak{X} \rightarrow \mathcal{Y}^{\bot}$, where $\bot$ represents the decision to defer to the human and $\mathcal{Y}^{\bot} = \mathcal{Y} \cup \{\bot\}$: $c_{\text{L2D}}\left(\delta(\rvx), \ry, \rh \right)  \triangleq   \left\{\ell(\delta(\rvx), \ry)  \text{ if } \delta(\rvx) \ne \bot; \ \ell(\rh, \ry)  \text{ if } \delta(\rvx) = \bot \right\}$.  If the decision rule returns a label $\ry \in \mathcal{Y}$, then we apply $\ell(\delta(\rvx), \ry)$, a loss between the classifier's prediction and the true label $\ry$.  When the decision is to defer, $\delta(\rvx) = \bot$, then we apply the same loss to the human's prediction: $\ell(\rh, \ry)$.  Assuming a 0-1 loss, the decision function that minimizes the conditional risk is: \begin{equation}\label{eq:l2d_defer}
    \delta_{0-1}^{*}\left(\rvx\right) \  \triangleq \   
\begin{cases}  
\bot  \ \text{ if } \ r^{*}(\vx)=1  \\
\argmax_{y} \mathbb{P}(\ry = y| \rvx) \ \text{ o/w }
\end{cases} r^{*}(\vx) \ \triangleq \ \mathbb{I}\left[\mathbb{P}(\rh = \ry | \vx) > \max_{y \in \mathcal{Y}} \mathbb{P}(\ry = y | \vx) \right]\end{equation} where $r^{*}(\vx)$ is a meta-classifier termed the \textit{rejector} and $\mathbb{P}(\rh = \ry | \rvx) = \sum_{y \in \mathcal{Y}} \mathbb{P}(\ry = y | \rvx) \cdot \mathbb{P}(\rh = y | \rvx)$---the human's (binary) probability of being correct.  If this probability is greater than the modal confidence of the Bayes predictor $\mathbb{P}(\ry| \rvx)$, deferring to the human is the best decision. Yet one may wonder how the human can outperform $\mathbb{P}(\ry| \rvx)$.  This is where additional information plays a role: the human can beat the Bayes predictor because of her access to a richer feature space $[\rvx, \rvz]$.  Much of the prior work on L2D has developed surrogate loss functions for the above learning problem \citep{pmlr-v119-mozannar20b}.  The most closely related work is by \citet{verma2022calibrated} and \citet{cao2023in}, who developed parameterizations that resulted in better confidence-calibrated estimates of $\mathbb{P}(\rh = \ry | \rvx)$.  We, on the other hand, will analyze L2D's calibration properties from a more general, theoretical perspective.


\looseness=-1
\subsection{Assumptions on Human and Model}
To finalize our setting of interest, we state the calibration assumption that will be the foundation of our analysis.  Following prior work in human-AI teaming \citep{kerrigan2021, pmlr-v119-mozannar20b}, we assume we do not have access to the human's (conditional) predictive distribution $\mathbb{P}(\rh | \rvx)$ \emph{nor can we model it well}.  If we could, the human could simply be replaced by a model.  Rather, we observe only samples, $\hat{\rh} \sim \mathbb{P}(\rh | \rvx)$.  Next we assume that both the human and model are calibrated for some arbitrary grouping function:
\begin{assumption}\label{fund_assump} \textbf{Pointwise Calibration of Human and Model.} Assume the model $\rvf$ is calibrated and measurable w.r.t.~grouping function $\Phi_{f}$: $\rvf(\rvx) = \mathbb{P}(\ry | \Phi_{f}(\rvx))$.  Likewise, assume the human's internal predictive model is calibrated and measurable w.r.t.~grouping function $\Phi_{h}$: $\mathbb{P}(\rh | \cdot) = \mathbb{P}(\ry | \Phi_{h}(\cdot))$, where $\Phi_{h}$ could partition $\mathfrak{X}$, $\mathcal{Z}$, or their combined feature space. 
\end{assumption}  Assuming the human provides only a sample $\hat{\rh}$ allows our analysis to proceed without the human providing calibrated confidence statements---which humans are typically not good at \citep{lichtenstein1982calibration, tetlock2016superforecasting}.  We stress that our Assumption \ref{fund_assump} does not assume the human or model are optimal predictors.  Rather, they are simply `unbiased with respect to each cell of the partition,' with Bayes optimality happening only when the cells become infinitesimal.


\looseness=-1
\section{Combination Through the Lens of Calibration}\label{sec:combination}We first study perhaps the simplest instantiation of teaming: pure model-based combination.  We then move on to the combination approach proposed by \citet{kerrigan2021}.  For both methods, we are interested if the combined human-AI predictor preserves calibration w.r.t.~$\Phi_{f}(\rvx)$ and $\Phi_{h}(\rvx)$.  


\looseness=-1
\subsection{Model-Based Combination}
Given the combination function $g$ from Section \ref{subsec:combo}, can $g$ be calibrated against $\Phi_h$ and / or $\Phi_f$?  If it is calibrated against both, then the human-AI team forms a calibrated predictor that successfully captures all `knowledge' encoded by $\Phi_h$ and $\Phi_f$.  Let the Bayes-optimal combination function be denoted $g^{*}(\hat{\rh}, \rvf(\rvx)) \triangleq \mathbb{P}(\ry \mid \hat{\rh},\, \rvf(\rvx))$.  Restricting our attention to $g^{*}$ produces the cleanest possible upper bound on what a calibrated $g$ can inherit.  Our result will show that $g^{*}$ cannot be calibrated w.r.t.~$\Phi_h$, as long as $\mathbb{P}(\rh | \Phi_{h}(\rvx))$ is not deterministic and $f(\rvx)$ does not `leak' information about $\Phi_h$.

\looseness=-1
\paragraph{Shared Feature Space}  We first consider the case of shared features: when the human and model operate on $\mathfrak{X}$. Consider the sampling step $\hat{\rh} \sim \text{\texttt{Categorical}}(\mathbb{P}(\rh | \Phi_{h}(\rvx)))$.  This step replaces a real-valued probability vector $\mathbb{P}(\rh | \Phi_{h}(\rvx))$ with a single categorical representation $\hat{\rh} \in \mathcal{Y}$. Of course, many independent draws would converge back to $\mathbb{P}(\rh | \Phi_{h}(\rvx))$, but the composite predictor has access to exactly one sample.  On the other hand, $g$ is allowed full access to $f$'s per-class confidences, so any partition $\Phi_{f}$ that is \emph{at least as coarse as} $f$'s output (e.g.~confidence calibration) is  visible to $g$. \begin{proposition} \textbf{Generalized Calibration of $g^{*}$ w.r.t.~$\Phi_{f}$.} \label{prop:cal_for_f} Assume $\Phi_f$ is a coarsening of $\rvf$'s level sets, i.e.\ $\sigma(\Phi_f) \subseteq \sigma(\rvf)$. Then $\rvg^{*}$ is calibrated w.r.t.~$\Phi_f$ in the sense of Definition \ref{def:comp_cal}.\end{proposition}\begin{proof}$\rvg^{*} = \mathbb{P}(\ry \mid \hat{\rh}, \rvf(\rvx))$ has input $\sigma$-algebra $\mathcal{S} = \sigma(\hat{\rh}, \rvf(\rvx))$, and
$\sigma(\Phi_{f}) \subseteq \sigma(\rvf) \subseteq \mathcal{S}$.  Since $\rvg^{*}$ and
$\Phi_{f}$ are both $\mathcal{S}$-measurable, the tower property gives
$\mathbb{P}(\ry \mid \rvg^{*}, \Phi_{f}) = \mathbb{E}[\,\mathbb{P}(\ry \mid \mathcal{S}) \mid \rvg^{*}, \Phi_{f}\,] = \rvg^{*} = \mathbb{P}(\ry \mid \rvg^{*})$.  Therefore $\rvg^{*}$ is calibrated w.r.t.~$\Phi_{f}$ (Definition \ref{def:comp_cal}).\end{proof}Proposition \ref{prop:cal_for_f} does not assert pointwise calibration w.r.t.~$\Phi_f$ (Definition \ref{def:cal}): $\rvg^{*}$ varies with $\hat{\rh}$ within each cell of $\Phi_f$, which is by design, to incorporate information from both upstream predictors. 

While calibration w.r.t.~the model's partitions is possible, calibration w.r.t.~the human's is not due to $\hat{\rh}$ being merely a categorical draw. \begin{theorem} \textbf{$g^{*}$ is not calibrated w.r.t.~$\Phi_h$.}
\label{thm:negative} Suppose there exist two cells $S_{1} \neq S_{2}$ of $\Phi_{h}$ and a value $\vv$ in the range of $\rvf$ such that: \begin{enumerate}[leftmargin=15pt, itemsep=0pt]
\item $\mathbb{P}(\rh | \Phi_{h}(\rvx))$ is non-degenerate: $\mathbb{P}(\rh = y' \mid \Phi_{h}(\rvx) = S_{i}) > 0$ for $i \in \{1, 2\}$, for some label $y'$. 
\item $\rvf$ does not separate $S_1$ and $S_2$: $\mathbb{P}(\rvf(\rvx) = \vv,\, \Phi_{h}(\rvx) = S_{i}) > 0$ for $i \in \{1, 2\}$.
\item $S_1$ and $S_2$ have different label information: $\mathbb{P}(\ry = y'' \mid \rvf(\rvx) = \vv, \Phi_{h}(\rvx) = S_{1}) \neq \mathbb{P}(\ry = y'' \mid \rvf(\rvx) = \vv, \Phi_{h}(\rvx) = S_{2})$, for some label $y''$.
\item $g^{*}$ refines $\rvf$: $\sigma(\rvf(\rvx)) \subseteq \sigma\left(g^{*}\left(\hat{\rh}, \rvf(\rvx)\right)\right)$.
\end{enumerate} Then $g^{*}(\hat{\rh}, \rvf(\rvx))$ is not calibrated with respect to $\Phi_h$. \end{theorem}  See Appendix \ref{sec:thm:negative} for the proof.  Conditions (1)--(3) imply that $\mathbb{P}(\ry \mid \hat{\rh}, \rvf(\rvx), \Phi_{h}(\rvx)) \neq \mathbb{P}(\ry \mid \hat{\rh}, \rvf(\rvx))$.  Condition (4) is required because if $g^{*}$ merges values with different model outputs, the label information that $\Phi_h$ provides at one of them can be averaged away by the others.  Condition (4) guarantees that the value of $g^{*}$ determines $\rvf(\rvx)$, mirroring the assumption from Proposition \ref{prop:cal_for_f} that $\sigma(\Phi_f) \subseteq \sigma(\rvf)$.  If $\rvf$ is multi-calibrated w.r.t.~$\sigma(\Phi_{f}, \Phi_{h})$, then $g$ can be calibrated w.r.t.~both.  This was prevented by condition (3).  While true multi-calibration effectively collapses $\rvg$ into $\rvf$, it suggests that $\rvf(\rvx)$ can `leak' information about the human's predictions to $g$.  This result is in the same spirit as \citet{Benz2023Aligned}'s solution of human-aligned calibration: improve the team via multi-calibrating the model to make it `human-aware.'  




\looseness=-1
\paragraph{Independent Features}  Now turning to the setting in which the human observes features $\rvz$ and the model observes $\rvx$, notice that Theorem \ref{thm:negative} immediately still applies.  Moreover, it can be strengthened by removing supposition (2), since it is impossible for $\rvf$ to separate the cells since $\rvf$ never even has access to $\rvz$.  Yet, ultimately, this is a good scenario and poses no additional complications to combination-based teaming.  Having access to additional information $\rvz$---if only by way of samples---still allows for $\rvg$ to demonstrate complementarity when $\mathbb{P}(\ry | \rvx, \rvz)$ is the true data generating process.


\looseness=-1
\paragraph{Simulation Study} We test these intuitions and empirically support our results via a simulation that compares the settings of shared and independent features.  The data generating distribution is set to:
\begin{equation*}
  \mathbb{P}(\ry \!=\! 1 \mid \rx, \rz) \;=\; \text{\texttt{logistic}}\!\left( -2 + \tfrac{3}{5}\,\Phi_{f}(\rx) + \tfrac{3}{5}\,\Phi_{h}(\rz) + \tfrac{1}{6}\, \Phi_{f}(\rx)\,\Phi_{h}(\rz) \right), 
\end{equation*} with the partitions being $\Phi_{f}(\rx), \Phi_{h}(\rz) \in \{0,1,2,3\}$.  In the \emph{shared-feature} setting, $(\rz = x_{1}, \rx = x_2 )$ such that $x_1 = x_2$ with probability $0.55$ and is drawn uniformly at random otherwise.  The partitions then evaluate to $\Phi_{h}(\rz) = x_1$ and $\Phi_{f}(\rx) = x_2$. In the \emph{independent} setting, $\rx$ and $\rz$ are independent and drawn uniformly on $\{0,1,2,3\}$, with $\Phi_{f}(\rx) = \rx$ and $\Phi_{h}(\rz) = \rz$. In both settings, the upstream predictors $\mathbb{P}(\rh | \rz)$ and $f(\rx)$ are calibrated against their respective partitions by construction. The downstream model $g^{*}$ is the Bayes-optimal predictor by setting it to the empirical mean of $\ry$ within each $\bigl(\hat{\rh}, f(\rx)\bigr)$ bucket. Figure~\ref{fig:model_comb_sim} reports reliability diagrams of $g^{*}$ within $\Phi_{f}$- and $\Phi_{h}$-partitions; marker area encodes cell count and color indexes the partition cell. The left plot of both subfigures shows that $g^{*}$ inherits $\Phi_{f}$-calibration in both regimes.  However, $g^{*}$ is miscalibrated against $\Phi_{h}$, with an ECE of $0.090$ in the shared-feature setting and an ECE of $0.144$ in the independent-feature setting.
\looseness=-1
\begin{figure}
    \centering
     \begin{subfigure}[b]{.85\textwidth}
         \centering
         \includegraphics[width=\linewidth]{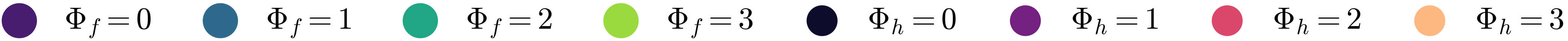}
     \end{subfigure}
\vfill
    \begin{subfigure}[b]{.47\textwidth}
	\centering
        \includegraphics[width=\textwidth]{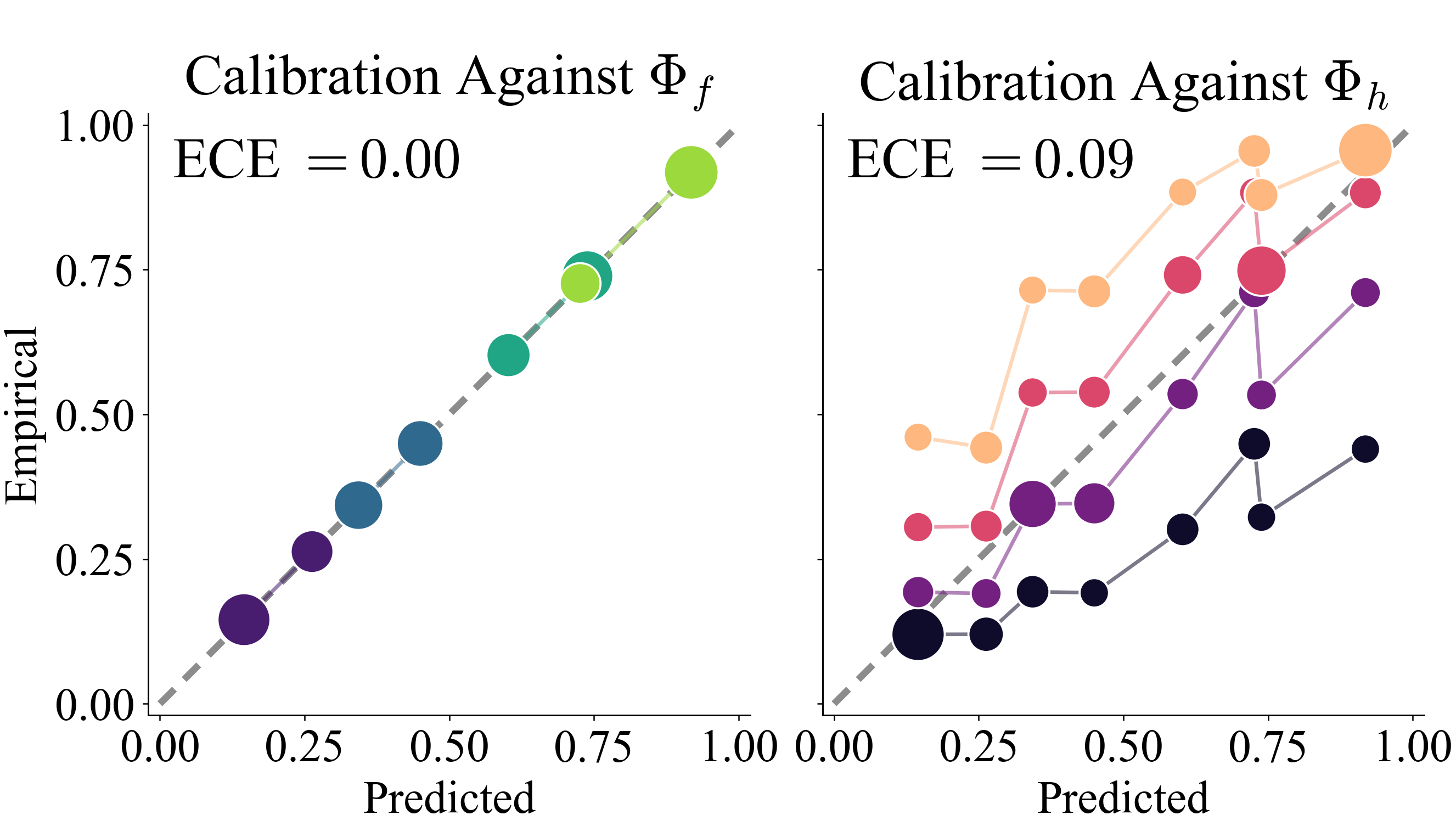} \\
        \subcaption{Shared Features}
        \label{fig:model_comb_shared_features}
    \end{subfigure}
    \hfill
    \begin{subfigure}[b]{.47\textwidth}
        \centering
        \includegraphics[width=\textwidth]{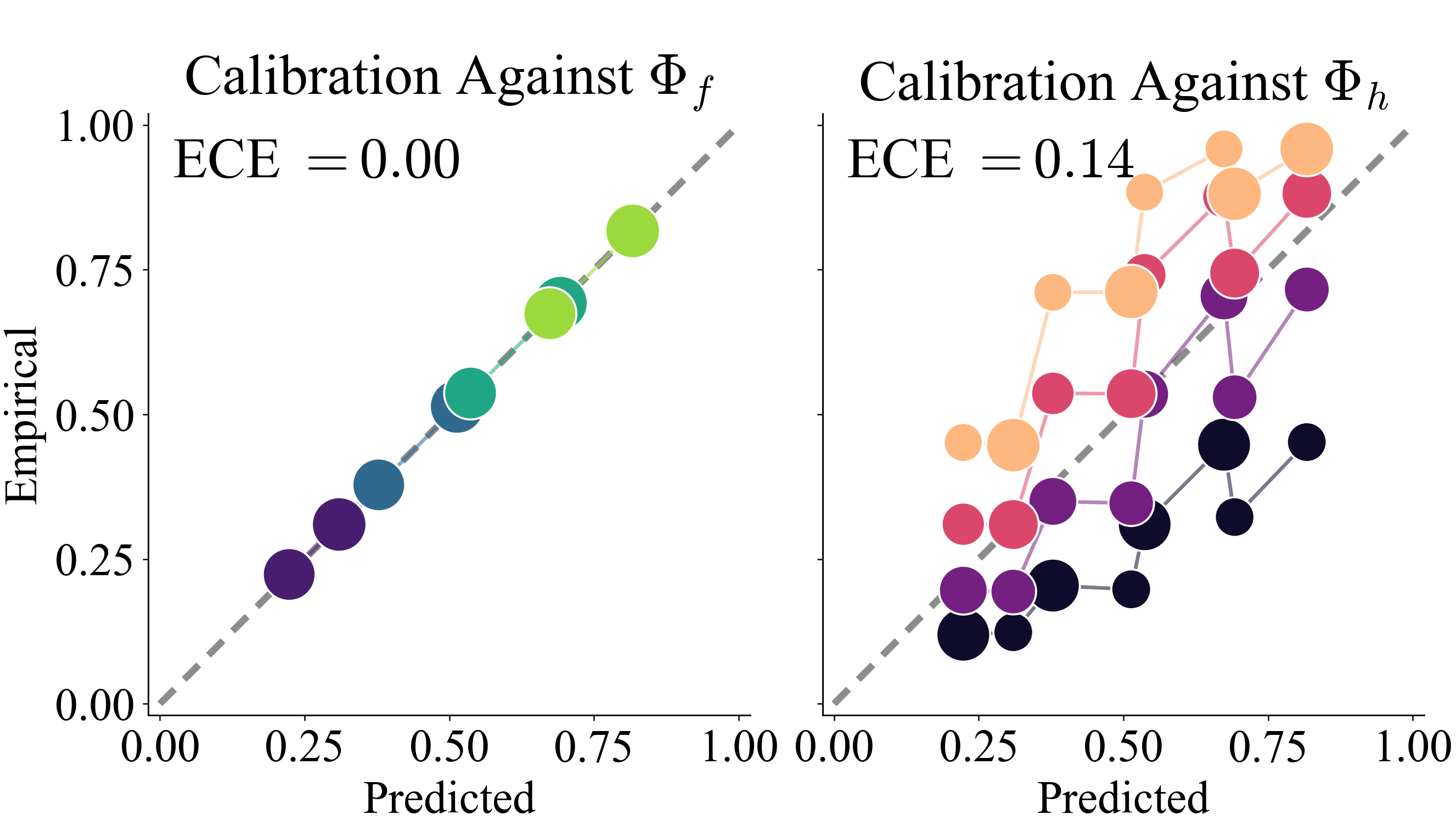} \\
        \subcaption{Independent Features}
        \label{fig:model_comb_indep_features}
    \end{subfigure}
    \caption{\textit{Calibration Failure of Model-Based Combination.}  The figures show simulations evaluating the calibration of $g^{*}$ in the case of a shared feature space (a) and independent feature spaces (b).  In both cases, the combination model $g^{*}$ is calibrated w.r.t.~$\Phi_{f}$ (ECE of 0) but not calibrated w.r.t.~$\Phi_{h}$ (ECE of $.09$ for shared, $.14$ for independent), the partitioning for which the human is calibrated.}  
    \label{fig:model_comb_sim}
\end{figure}

\looseness=-1
\subsection{Bayesian Combination via Confusion Matrix Representation of Human}\label{sec:bayes_comb_results}  We now turn to the Bayesian combination method of \citet{kerrigan2021}.  Unlike the previous sub-section, where the combination was limited to a single sample from the human, this combination sees a (normalized) confusion matrix that aggregates the human's (sampled) predictions. 


\looseness=-1
\paragraph{Shared Feature Space}  We return to the setting of no hidden features: the model and human both make predictions from $\rvx$.  Like in the previous combination setting, we are interested if the combined model is calibrated w.r.t.~$\Phi_{f}$ and / or $\Phi_{h}$.  For the former, we again have an affirmative result.
\begin{theorem}\label{thm_cal_bayes_f} \textbf{Bayesian combination posterior is calibrated w.r.t.~$\Phi_{f}$.}  Assuming that $\Phi_{h} \ \bot \ \Phi_{f}(\rvx) \ | \ \ry$, the posterior distribution is calibrated w.r.t.~$\Phi_f$ in the sense of Definition \ref{def:comp_cal}:\begin{equation*}
    p\left(\ry | \rh, f(\rvx) \right) \ = \ \mathbb{P}\left(\ry \ | \ \rh, \Phi_{f}(\rvx) \right) \ = \ \frac{\mathbb{E}_{\Phi_{h}}\left[ \mathbb{P}(\rh | \Phi_{h}(\rvx)) | \ry \right] \cdot \mathbb{P}\left(\ry | \Phi_{f}(\rvx) \right)}{\mathbb{P}(\rh | \Phi_{f}(\rvx))}
\end{equation*}\end{theorem}  See Appendix \ref{thm_cal_bayes_f_proof}  for the proof.  The result is made straightforward by the conditional independence assumption since it implies $\rh \ \bot \ \Phi_{f}(\rvx) \ | \ \ry$, which then allows for the confusion matrix representation of $p(\rh | \ry)$.  This assumption is crucial, and we will revisit it below.  Turning to the human, again we have a negative result for $\Phi_h$. \begin{proposition}\textbf{Bayesian combination posterior is not calibrated w.r.t.~$\Phi_{h}$.}\label{thm:neg_bayes_h} Assume that $\Phi_{h}$ and $\rvf$ satisfy conditions (1)--(3) of Theorem \ref{thm:negative}.  Moreover, assume that $p(\ry \mid \rh, \rvf(\rvx))$ satisfies condition (4) in place of $g^{*}$ in Theorem \ref{thm:negative}.  Then the posterior is not calibrated with respect to $\Phi_{h}$. \end{proposition} \begin{proof}
Like $g^{*}$, the confusion-matrix posterior is a measurable function of $(\rh, \rvf(\rvx))$.  Yet the human enters only through the confusion matrix $\rvc_{\ry,\cdot} = \int \mathbb{P}(\rh \mid \Phi_{h}(\rvx))\, d\mathbb{P}(\rvx \mid \ry)$, which averages out the within-class variation of $\Phi_{h}$.  Therefore the proof of Theorem \ref{thm:negative} is applicable here: conditions (1)--(3) make $\Phi_{h}$ carry label information beyond $(\rh, \rvf(\rvx))$, and condition (4) ensures that label information cannot be marginalized away. Thus the definition of generalized calibration (\ref{def:comp_cal}) fails.
\end{proof} From these results, we see that the best case scenario is when $\Phi_{f}$ and $\Phi_{h}$ are conditionally independent (given $\ry$), thereby encoding unique aspects of the feature space.  However, the model becomes quite sensitive to this property, to the point of being non-robust.  The \emph{finer grained} $\Phi_{f}$ becomes, the more likely the conditional independence assumption is to be broken.  The posterior tends to over-sharpen via double-counting: the posterior will not be ideally calibrated even if $\rvf(\rvx) = \mathbb{P}(\ry | \rvx)$.
\begin{proposition}\label{prop_post_miscal_f_bayes}\textbf{Posterior miscalibration when model is already Bayes predictor.}  Let $\Phi_{f}(\rvx) = \rvx$ such that $\rvf(\rvx) = \mathbb{P}(\ry \mid \rvx)$, and assume that $\mathbb{P}(\ry \mid \rvx)$ has full support on $\mathcal{Y}$ on a set of $\rvx$ of positive probability.  Suppose that the human's prediction is informative, meaning it is not independent of the label, $\rh \not\perp \ry$.  Then the posterior is not ideally  calibrated: $p\left(\ry \ | \ \rh, \rvf(\rvx) \right) \ne \mathbb{P}(\ry | \rvx)$. \end{proposition} 
See Appendix \ref{sec:prop_post_miscal_f_bayes} for the proof.  This posterior sharpening does \emph{not} occur when $\mathbb{P}(\rh \mid \rvx)$ is a uniform distribution, which necessitates the supposition that the human is informative.  This result presents an interesting contrast to model-based combination, which \emph{can} be ideally calibrated when the model is.

\looseness=-1
\paragraph{Independent Features} In the case of independent features, the human and model are guaranteed to introduce complementary information: $\rvz \perp \rvx \ | \ \ry$ implies $\Phi_{h} \perp \Phi_{f} \ | \ \ry$.  In turn, while the posterior is still not calibrated w.r.t.~$\Phi_{h}$, we no longer have to worry about $f(\rvx)$ being \emph{too calibrated} and combination introducing over-sharpening.  This suggests an experimental hypothesis: when the model and human share features, we should see the posterior become an increasingly \emph{worse} predictor as $\rvf$ becomes more calibrated.  However, with independent feature spaces, improvements in $\rvf$'s calibration should always translate to a better posterior predictor.

\begin{figure}[t]
    \centering
    \begin{subfigure}[b]{0.329\textwidth}
        \centering
        \includegraphics[width=\linewidth]{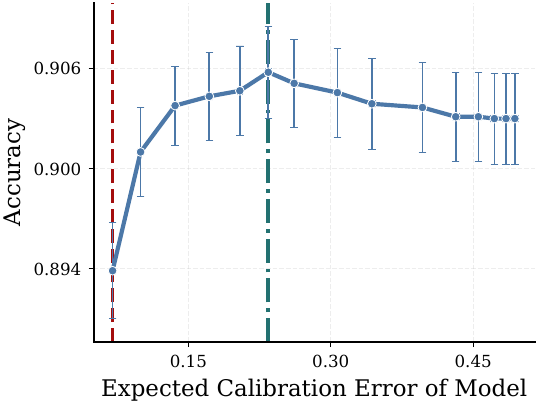}
        \caption{ImageNet-16H: Accuracy}
        \label{fig:imagenet16h_comb_acc}
    \end{subfigure}
    \begin{subfigure}[b]{0.329\textwidth}
        \centering
        \includegraphics[width=\linewidth]{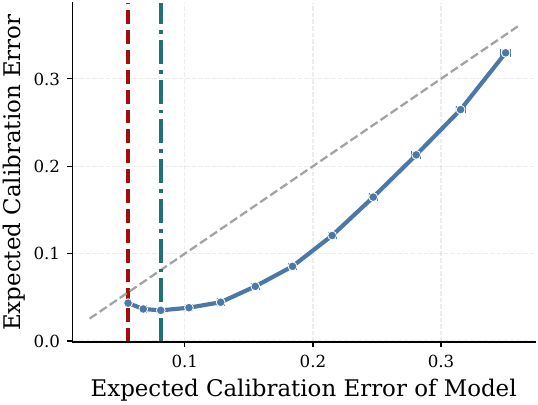}
        \caption{ImageNet-16H: ECE}
        \label{fig:imagenet16h_comb_ece}
    \end{subfigure}
    \begin{subfigure}[b]{0.329\textwidth}
        \centering
        \includegraphics[width=\linewidth]{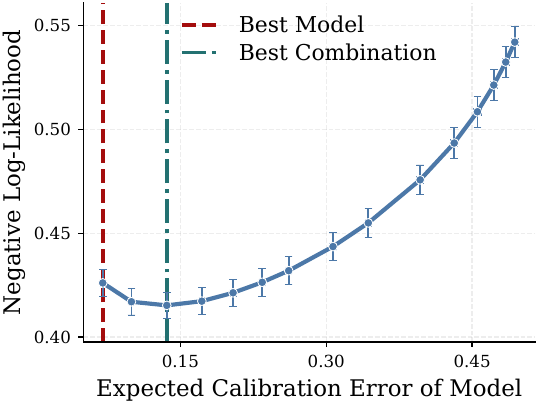}
        \caption{ImageNet-16H: Neg. LL}
        \label{fig:imagenet16h_comb_nll}
    \end{subfigure}
    \vspace{0.8em}
    \begin{subfigure}[b]{0.329\textwidth}
        \centering
        \includegraphics[width=\linewidth]{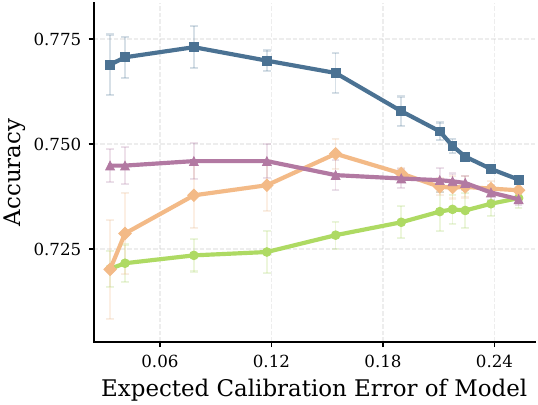}
        \caption{HAM10000: Accuracy}
        \label{fig:ham10k_acc}
    \end{subfigure}
    \begin{subfigure}[b]{0.329\textwidth}
        \centering
        \includegraphics[width=\linewidth]{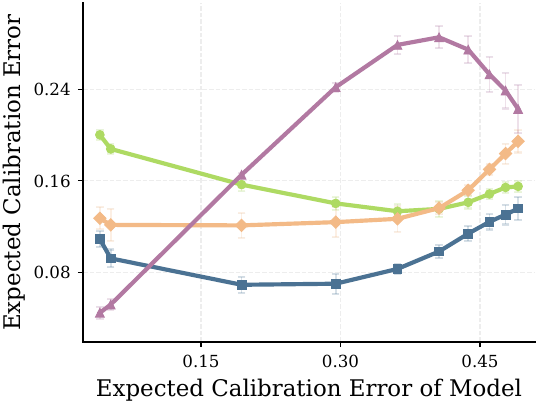}
        \caption{HAM10000: ECE}
        \label{fig:ham10k_ece}
    \end{subfigure}
    \begin{subfigure}[b]{0.329\textwidth}
        \centering
        \includegraphics[width=\linewidth]{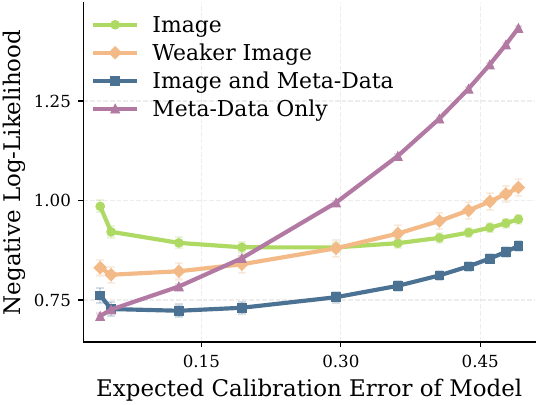}
        \caption{HAM10000: Neg. LL}
        \label{fig:ham10k_nll}
    \end{subfigure}
    \caption{\textit{Empirical Validation for Real and Simulated Humans.}
    \textit{Top row}: For ImageNet-16H, we vary the (confidence) calibration of a VGG-19 classifier via temperature scaling and combine each classifier with a sampled human label using the confusion-matrix update.  The best calibrated model never results in the best posterior predictor. \textit{Bottom row}: On HAM10000, we vary the calibration of an image-only ResNet-18 classifier and combine it with sampled predictions from simulated humans with a range of feature spaces.  When `humans' see only meta-data, better model calibration results in an improved posterior predictor.  Error bars show one standard error over five random seeds.}
    \label{fig:empirical_validation}
\end{figure} \looseness=-1
\paragraph{Experimental Validation on ImageNet-16H}  We evaluate the above hypothesis starting with the \textit{shared features} setting.  We use the same experimental setup as \citet{kerrigan2021} for the ImageNet-16H dataset, which contains $28,997$ human classifications of ImageNet-16 images from a total of $145$ participants.\footnote{While the human participants could be using hidden features to inform their predictions, the fact that computer vision systems have surpassed human performance on ImageNet \citep{superhuman2015ICCV} suggests any advantage from hidden information is negligible.}  For each image, the dataset provides approximately six human annotations. We aggregate these annotations and interpret the normalized count vector as a human's predictive distribution for that image. To match our theoretical setting, we sample from this empirical distribution. See Appendix \ref{bayes_comb_exper_details} for additional experimental details.  We evaluate the performance of the resulting human-AI combination as the classifier's individual confidence calibration varies.  We do this via temperature scaling the VGG-19 softmax probabilities \citep{guo2017calibration}.  We sweep over a grid of temperatures and evaluate both the standalone classifier and the confusion-matrix-combined predictor using held-out samples.  We report three quantities: accuracy, expected (confidence) calibration error (ECE), and negative log-likelihood (NLL).  Results are reported in Subfigures \ref{fig:imagenet16h_comb_acc}, \ref{fig:imagenet16h_comb_ece}, and \ref{fig:imagenet16h_comb_nll}.  For a shared feature space, better model calibration does not lead to a better combined predictor.  We report similar trends for CIFAR-10H in Appendix \ref{app:cifar10h_validation}. 

\looseness=-1
\paragraph{Experimental Validation on HAM10000} We next use HAM10000 \citep{tschandl2018ham10000} to analyze the \textit{independent feature} setting.  Each example contains a $224 \times 224$ dermoscopic image along with patient- or lesion-level meta-data. We identify the image with the model-observed feature $\rvx$ and treat the meta-data as $\rvz$---unavailable to the base classifier but available to the human.  The model is a temperature-scaled ResNet-18.  We simulate the human using an auxiliary classifier---following \citet{verma2022calibrated}---allowing for a controlled comparison of several feature-access settings. The first, denoted \emph{Image}, represents the human via the same ResNet-18 employed for the main classifier (i.e.~perfectly redundant features). The second, denoted \emph{Weaker Image}, uses a weaker ResNet-18 trained on $128 \times 128$ images for 4 epochs, which represents a human with a coarser feature space. The third, denoted \emph{Image and Meta-Data}, is a multimodal expert that uses a ResNet-18  to encode the image before combining it with meta-data via an MLP.  This represents a human with a strictly richer feature space. The fourth, denoted \emph{Meta-Data Only}, is a logistic-regression-based human trained on age, sex, and lesion localization.  It has an independent feature space, unlike the previous three (simulated) humans.  See Appendix \ref{bayes_comb_exper_details} for additional details.  Subfigures \ref{fig:ham10k_acc}, \ref{fig:ham10k_ece}, and \ref{fig:ham10k_nll} report accuracy, ECE under confidence calibration, and NLL, respectively.  Just as for the ImageNet-16H results, the y-axis evaluates the predictive quality of the posterior as a function of $\rvf(\rvx)$'s confidence calibration by varying the temperature parameter.  For the meta-data-only human (purple), we see that---as our theory predicts---the posterior achieves its best accuracy, ECE, and NLL for the \emph{best-calibrated} classifier.  For the image-only human, the posterior accuracy is best for the \emph{worst calibrated model}, and the posterior ECE and NLL are best for the fifth-best calibrated model.  For the image-and-meta-data human (blue), the posterior is best for the third-best calibrated model.  See Appendix \ref{app:noisy_biased_human} for additional experiments simulating humans with noisy and biased predictions.


\section{Delegation Through the Lens of Calibration}\label{sec:delegation}

\looseness=-1
We now turn to the alternative setting of delegation: either the human or model predicts with the selection being determined by the rejector $r(\rvx)$.  Since the model or human alone makes the prediction, predictions will always come from a predictor that is either $\Phi_f$- or $\Phi_h$-calibrated, by assumption.  Thus the focus of our analysis turns to $r(\rvx)$: \textit{how must the rejector be calibrated in order to choose the superior downstream predictor?}  Assuming the downstream predictors are fixed, we can write the combined (system) risk as:  \begin{equation}\label{eq:cal_01_risk}
    \mathcal{R}^{0\text{-}1}_{f,h}(r) \ \triangleq \
    \mathbb{E}_{\rvx}\!\left[(1 - r(\rvx))\!\left(1 - \max_{y} f_{y}(\rvx)\right)
    + r(\rvx)\!\left(1 - \mathbb{P}(\rh = \ry \mid \rvx)\right)\right].
\end{equation}  Minimization of Equation \ref{eq:cal_01_risk} yields the oracle rule $r^{*}(\rvx) = \mathbb{I}[\mathbb{P}(\rh = \ry \mid \rvx) > \max_{y} f_{y}(\rvx)]$, recovering the Bayes-optimal deferral rule of \citet{pmlr-v119-mozannar20b}.  In practice, we assume the rejector is parameterized by a scalar score $\rrho: \mathfrak{X} \times \Delta^{K-1} \mapsto [0,1]$ such that \begin{equation}\label{eq:l2d_decision_rule}
    r(\rvx; \rrho) \ \triangleq \ \mathbb{I}\!\left[\rrho(\rvx, \rvf(\rvx)) \ > \ \max_{y} f_{y}(\rvx)\right].
\end{equation}  If $\rrho$ is ideally calibrated such that $\rrho(\rvx, \rvf(\rvx)) = \mathbb{P}(\rh = \ry \mid \rvx)$, then the oracle rule $r^{*}$ is recovered.  Notably, $\rrho$ takes both the original features $\rvx$ and the model's output $\rvf(\rvx)$ as inputs.  This explicit dependence on $\rvf(\rvx)$ is a key feature of this parameterization and differs from most L2D implementations \citep{okati2021differentiable, pmlr-v119-mozannar20b, verma2022calibrated, cao2023in}---but not all \citep{madras2018predict}---where the probability that the expert is correct is typically estimated from $\rvx$ alone.  This choice does not incur computational overhead since finding the maximum dimension of $\rvf(\rvx)$ is needed to evaluate $r(\rvx; \rrho)$.  We discuss the theoretical implications of this choice below.

\looseness=-1
We now present our central result for delegation that characterizes the rejector's calibration requirements in relation to the downstream predictors.
\begin{theorem}\textbf{Risk-minimizing rejector under calibration.}\label{thm:cal_rejector}  Assume ties between the downstream predictors are negligible: $\mathbb{P}(\rh = \ry \mid \rvx) \neq \max_{y} f_{y}(\rvx)$ almost surely.  Let $\Phi_{r}$ be a grouping function on $\mathfrak{X} \times \Delta^{K-1}$, and let $\tilde{\Phi}_{r}(\rvx) \triangleq \Phi_{r}(\rvx, \rvf(\rvx))$ be the partition of $\mathfrak{X}$ induced by the additional refinement w.r.t.~$\rvf(\rvx)$.  Assume $\rrho$ is calibrated w.r.t.~$\tilde{\Phi}_{r}(\rvx)$.  Then the rejector $r(\rvx; \rrho)$ (Equation \ref{eq:l2d_decision_rule}) minimizes
$\mathcal{R}^{0\text{-}1}_{f,h}(r)$ (Equation \ref{eq:cal_01_risk}) if and only if  $\sigma(\1_{\mathfrak{D}^{*}}) \subseteq \sigma(\tilde{\Phi}_{r})$, where $\mathfrak{D}^{*}  \triangleq  \{ \rvx \in \mathfrak{X} \ | \ r^{*}(\rvx) = 1 \}$, the points for which the human should make the prediction under the Bayes optimal rejector.
\end{theorem} See Appendix~\ref{proof_cal_rejector} for the proof.  As one would expect, the rejector's induced partition $\tilde{\Phi}_{r}$ must be fine-grained enough to identify the optimal-deferral set $\mathfrak{D}^{*}$.  Thus one pressure on $\tilde{\Phi}_{r}$ is to capture which side of the threshold $\max_{y} f_{y}$ the human's correctness probability falls on.  As the human's partition $\Phi_{h}$ becomes finer-grained, $\tilde{\Phi}_{r}$ must increase its granularity to track its threshold-crossings.  As $\Phi_{h}$'s partitions shrink in size, $\sigma(\1_{\mathfrak{D}^{*}})$ must approach the resolution of $\sigma(\Phi_{f}, \Phi_{h})$ at the boundaries of $\mathfrak{D}^{*}$.  This would force $\tilde{\Phi}_{r}$ to absorb $\Phi_{h}$ and may make the human's involvement unnecessary (as her predictive power becomes increasingly captured by $\rrho$ alone).  

\looseness=-1
\paragraph{$\rvf(\rvx)$ as Input to the Rejector} In Theorem \ref{thm:cal_rejector}, the model's partition $\Phi_{f}$ does not appear explicitly, despite the comparison to $\max_{y} f_{y}(\rvx)$ being $\Phi_{f}$-measurable.  This requirement is circumvented by giving $\rrho$ access to $\rvf(\rvx)$ by assumption, which forces $\tilde{\Phi}_{r}$ to refine $\Phi_{f}$ by construction.  If $\rrho$ had $\rvx$ as its only input, then the sufficient condition would become $\sigma(\Phi_{f}, \1_{\mathfrak{D}^{*}}) \subseteq \sigma(\Phi_{r})$: $\Phi_{r}$ would need to explicitly refine $\Phi_{f}$ in addition to identifying $\mathfrak{D}^{*}$.  As mentioned above, adding $\rvf(\rvx)$ as a feature to the rejector is not common in L2D implementations but actually was done in \citet{madras2018predict}'s seminal work, with the choice being justified through intuition rather than theory.  The majority of subsequent implementations use a neural network to map $\rvx$ to $K+1$ outputs \citep{pmlr-v119-mozannar20b, verma2022calibrated, cao2023in}, with the classifier encoded by the first $K$ dimensions and the rejector encoded by the $(K+1)$th.  When the classifier and rejector share a backbone architecture, we expect adding $\rvf(\rvx)$ as a feature to matter less for these implementations than ones with independent parameterizations \citep{madras2018predict, okati2021differentiable}.


\looseness=-1
\paragraph{Human with Additional Information}  We next consider the case in which the human uses $\rvx$ and additional information $\rvz$ to make her prediction (calibrated w.r.t.~$\Phi_{h}(\rvx, \rvz)$).  This is the setting assumed by L2D's theoretical foundation, since $\rvz$ is necessary for the human to outperform the Bayes classifier on $\mathfrak{X}$.  Yet, like the model, the rejector only has access to $\rvx$, and it must decide which downstream predictor to use without knowing $\rvz$ or $\rh$.  The risk in Equation \ref{eq:cal_01_risk} then becomes
\begin{equation}\label{eq:cal_01_risk_z}
    \widetilde{\mathcal{R}}^{0\text{-}1}_{f,h}(r) \ = \ \mathbb{E}_{\rvx, \rvz}\!\left[(1 - r(\rvx))\!\left(1 - \max_{y} f_{y}(\rvx)\right) + r(\rvx)\!\left(1 - \mathbb{P}(\rh = \ry \mid \rvx, \rvz)\right)\right].
\end{equation} Without access to the additional information, Equation \ref{eq:cal_01_risk_z} can be minimized only up to marginalization over $\rvz$: $\mathbb{P}(\rh = \ry \mid \rvx) = \mathbb{E}_{\rvz \mid \rvx}[\mathbb{P}(\rh = \ry \mid \rvx, \rvz)]$.  Theorem \ref{thm:cal_rejector} then immediately yields:
\begin{corollary}\textbf{Best $\mathfrak{X}$-measurable rejector under additional information.}\label{prop:rejector_z} Let $\tilde{\mathfrak{D}} \triangleq \left\{\rvx \in \mathfrak{X} \mid \mathbb{P}(\rh = \ry \mid \rvx) > \max_{y} f_{y}(\rvx)\right\}$ denote the $\rvz$-marginalized optimal-deferral set.  Among $\mathfrak{X}$-measurable rejectors, $r(\rvx; \rrho)$ minimizes
$\widetilde{\mathcal{R}}^{0\text{-}1}_{f,h}(r)$ (Equation \ref{eq:cal_01_risk_z}) if and only if $\sigma(\1_{\tilde{\mathfrak{D}}}) \subseteq \sigma(\tilde{\Phi}_{r})$.
\end{corollary}
\begin{wrapfigure}{r}{0.35\linewidth}
    \centering
    \includegraphics[width=\linewidth]{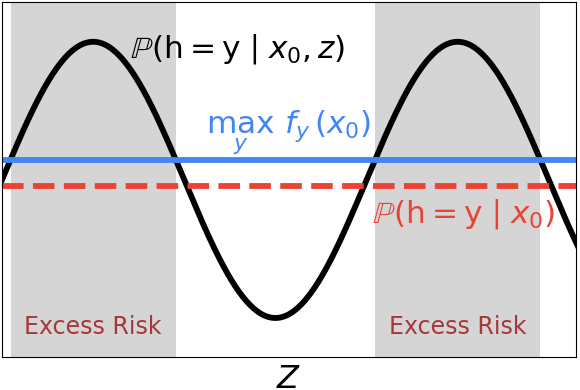}
    \caption{\textit{Threshold crossings due to variation in $\rvz$ drive excess risk.}}
\vspace{-18pt}
\end{wrapfigure}
\looseness=-1 The best $\mathfrak{X}$-measurable rejector cannot in general match the oracle $r^{*}(\rvx, \rvz) = \mathbb{I}[\mathbb{P}(\rh = \ry \mid \rvx, \rvz) > \max_{y} f_{y}(\rvx)]$.  The excess risk $\widetilde{\mathcal{R}}^{0\text{-}1}_{f,h}(r(\rvx; \rrho^{*})) - \widetilde{\mathcal{R}}^{0\text{-}1}_{f,h}(r^{*}(\rvx, \rvz))$ vanishes only when---for each $\rvx = \vx_0$---$\mathbb{P}(\rh = \ry \mid \vx_0, \rvz)$ and $\mathbb{E}_{\rvz \mid \vx_0}[\mathbb{P}(\rh = \ry \mid \vx_0, \rvz)]$ always remain on the same side of the threshold $\max_{y} f_{y}(\vx_0)$.  An example of a setting with excess risk is shown in the figure to the right.  We see that $\mathbb{P}(\rh = \ry \mid \vx_0, \rvz)$ (black line) crosses the threshold  $\max_{y} f_{y}(\vx_0)$ (blue line) several times, yet $\mathbb{E}_{\rvz \mid \vx_0}[\mathbb{P}(\rh = \ry \mid \vx_0, \rvz)]$ (red line) always signals to not defer to the human.  Of course, multi-calibrating the rejector w.r.t.~$\Phi_{h}$ would remove the information asymmetry and the resulting excess risk.  In summary, the presence of additional information is easy to incorporate via human-AI combination and presents no fundamental modeling challenge.  For delegation, the story is not so simple: the deferral policy is fundamentally limited by the unobservable variation in $\rvz$ and its effect on whether the human is the superior downstream predictor.



\looseness=-1
\section{Discussion, Limitations, Future Work}
\looseness=-1  
Assuming a human and model are calibrated predictors, we have analyzed how these calibration properties propagate into the human-AI team's construction.  In particular, we studied two methods for combining human and model predictions and the standard approach to delegating predictions (i.e.~learning to defer).  For the combination-based approaches, calibration w.r.t.~the model's partition $\Phi_f$ is possible (\ref{prop:cal_for_f}, \ref{thm_cal_bayes_f}), but calibration w.r.t.~the human's partition $\Phi_h$ is not (\ref{thm:negative}, \ref{thm:neg_bayes_h}).  This latter result was due to the human's prediction either being sampled or aggregated in a confusion matrix---both of which unavoidably destroy the information in $\Phi_h$.  The possible strategies for regaining $\Phi_h$-calibration are to multi-calibrate the model or draw multiple human predictions \citep{vul2008measuring}, which will recover $\mathbb{P}(\rh | \Phi_{h}(\rvx))$ asymptotically.  The presence of additional features only available to the human poses no issue to the combination frameworks we studied.  Specifically, \citet{kerrigan2021}'s combination method works better when there is information asymmetry, since stronger model calibration results in a better posterior predictor.  Figure \ref{fig:empirical_validation} reports experimental verification of this result.  

\looseness=-1 
In regards to delegation, Theorem \ref{thm:cal_rejector} shows that the optimal rejector must be calibrated with a granularity determined by the optimal-deferral set---$\{ \rvx \in \mathfrak{X} \ | \ \mathbb{P}(\rh = \ry \mid \rvx) > \max_{y} f_{y}(\rvx) \}$---and $\Phi_{f}$.  This latter requirement suggests that L2D implementations should always include $\rvf(\rvx)$ as an input feature to the rejector model, providing refinement w.r.t.~$\Phi_{f}$ by construction.  Unfortunately, the presence of hidden information makes delegation difficult: the deferral policy can at best make decisions considering only the $\rvz$-marginalized expert correctness, $\mathbb{E}_{\rvz \mid \rvx}[\mathbb{P}(\rh = \ry \mid \rvx, \rvz)]$.  Therefore human-AI combination is a more natural solution in the presence of considerable feature asymmetry.

\looseness=-1
\paragraph{Limitations} The primary limitation of this work is that the model and human are assumed to be calibrated, when, in practice, calibration can be hard to verify for even modestly fine-grained partitions (e.g.~confidence calibration) \citep{NEURIPS2019_f8c0c968, pmlr-v291-rossellini25a}.  Yet, since our results mostly highlight limitations in teaming frameworks, assuming the good-case scenario of calibration increases their generality. Furthermore, learning has been avoided entirely, and as the considerable work on learning theory for delegation has shown, issues surrounding parameterization, consistency and realizability matter in practice \citep{charusaie2022sample, mozannar2023exact, maomastering, mohri2024learning, mao2024realizable, okati2021differentiable}.  Lastly, while experiments on ImageNet-16H and HAM10000 supported our results in Section \ref{sec:bayes_comb_results}, our other results were empirically validated only via small-scale simulations. 

\looseness=-1
\paragraph{Future Work}  It would be natural to extend our results to teams with multiple humans and AI(s) \citep{keswani2021towards, verma2023learning, mao2023twostage, tailor2024learning, hemmer2022forming}.  Adapting our results to support approximate calibration \cite{foster1998asymptotic,NEURIPS2021_bbc92a64,10.1145/3564246.3585182,gopalan2025efficient} might allow for analysis of combination methods that take as input a human's miscalibrated confidences.  Lastly, studying human-AI teaming over time (possibly supported by a priority queue \citep{lykouris2024learning}), such that the human and model can interact and adapt \citep{Bansal2019Updates, noti2025ai, benz2026learning, 10.1093/pnasnexus/pgag030}, introduces complexity that is outside the scope of static calibration assumptions and is highly relevant to modern human-AI teaming systems.

\looseness=-1
\section*{Acknowledgments} We acknowledge computational resources from the Johns Hopkins DSAI cluster.  Yixin Wang was supported in part by funding from the Office of Naval Research under grant N00014-23-1-2590, the National Science Foundation under grants No. 2310831, No. 2428059, No. 2435696, No. 2440954, the Michigan Institute for Data Science's Propelling Original Data Science (PODS) grant, Two Sigma Investments LP, and LG Management Development Institute AI Research.

\bibliographystyle{plainnat}
\bibliography{references}

\appendix

\section{Proof of Theorem \ref{thm:negative}} \label{sec:thm:negative}
\begin{proof}
The human's sample satisfies $\hat{\rh} \ \bot \ (\ry, \rvf(\rvx)) \mid \Phi_{h}(\rvx)$.  By this independence together with conditions (1) and (2), the shared evidence value $(\hat{\rh}, \rvf(\rvx)) = (y', \vv)$ has positive probability under both cells:
\begin{equation*}
    \mathbb{P}\!\left(\hat{\rh} = y', \rvf(\rvx) = \vv, \Phi_{h}(\rvx) = S_{i}\right)
    \ = \ \underbrace{\mathbb{P}\!\left(\hat{\rh} = y' \mid \Phi_{h}(\rvx) = S_{i}\right)}_{> 0 \text{ by (1)}}\,\,
    \underbrace{\mathbb{P}\!\left(\rvf(\rvx) = \vv, \Phi_{h}(\rvx) = S_{i}\right)}_{> 0 \text{ by (2)}} \ > \ 0 .
\end{equation*}
On this set we have, for $i = 1, 2$,
\begin{equation}\label{eq:neg_within_cell}
    \mathbb{P}\!\left(\ry = y \mid \hat{\rh} = y', \rvf(\rvx) = \vv, \Phi_{h}(\rvx) = S_{i}\right)
    \ = \ \mathbb{P}\!\left(\ry = y \mid \rvf(\rvx) = \vv, \Phi_{h}(\rvx) = S_{i}\right),
\end{equation}
since conditioning further on $\hat{\rh} = y'$ adds no information about $\ry$
given $(\Phi_{h}, \rvf)$.  The two sides of Equation \ref{eq:neg_within_cell} differ for
$i = 1, 2$ by the differing-content requirement of (3).  Therefore conditioning
additionally on $\Phi_{h}(\rvx)$ changes the label probability beyond
$(\hat{\rh}, \rvf(\rvx))$, i.e.
\begin{equation*}
    \mathbb{P}\!\left(\ry \mid \rvf(\rvx), \hat{\rh}, \Phi_{h}(\rvx)\right) \ \neq \ \mathbb{P}\!\left(\ry \mid \rvf(\rvx), \hat{\rh}\right)
\end{equation*}
on this positive-probability set. 

Write $c \triangleq g^{*}(y', \vv)$ and $H_{c} \triangleq \{h \in \mathcal{Y} \mid g^{*}(h, \vv) = c\} \ni y'$.  By condition (4), every value with $g^{*} = c$ has model output $\vv$, so $\{g^{*} = c\} = \{\hat{\rh} \in H_{c}, \ \rvf(\rvx) = \vv\}$.  Since $\hat{\rh} \perp (\ry, \rvf(\rvx)) \mid \Phi_{h}(\rvx)$, conditioning on $\hat{\rh} \in H_{c}$ adds no information about $\ry$ given $(\rvf(\rvx), \Phi_h(\rvx))$, and therefore, for $i = 1, 2$,
\begin{equation*}\begin{split}
    \mathbb{P}\!\left(\ry \mid g^{*} = c, \Phi_{h}(\rvx) = S_{i}\right) \ &= \ \mathbb{P}\!\left(\ry \mid \hat{\rh} \in H_{c}, \rvf(\rvx) = \vv, \Phi_{h}(\rvx) = S_{i}\right) \\ &= \ \mathbb{P}\!\left(\ry \mid \rvf(\rvx) = \vv, \Phi_{h}(\rvx) = S_{i}\right),
\end{split}
\end{equation*}
where both conditioning events have positive probability, because $\mathbb{P}(\hat{\rh} \in H_{c} \mid \Phi_h(\rvx) = S_i) \geq \mathbb{P}(\hat{\rh} = y' \mid \Phi_h(\rvx) = S_i) > 0$ by (1) and $\mathbb{P}(\rvf(\rvx) = \vv, \Phi_h(\rvx) = S_i) > 0$ by (2).  By (3), the right-hand sides differ for $i = 1$ and $i = 2$ at the label $y''$, so they cannot both equal $\mathbb{P}(\ry = y'' \mid g^{*} = c)$.  Hence $\mathbb{P}(\ry \mid g^{*}, \Phi_{h}) \neq \mathbb{P}(\ry \mid g^{*})$ on a set of positive probability, which is exactly the failure of generalized calibration (Definition \ref{def:comp_cal}).  \end{proof}

\section{Proof of Theorem \ref{thm_cal_bayes_f}}

\label{thm_cal_bayes_f_proof}


\begin{proof}
Fix a label value $\ry=y$ and a realized human prediction $\rh=h$. 
By Bayes rule,
\begin{equation*}
    \mathbb{P}\left(\ry=y \mid \rh=h, \Phi_f(\rvx)\right)
    =
    \frac{
        \mathbb{P}\left(\rh=h \mid \ry=y, \Phi_f(\rvx)\right)
        \mathbb{P}\left(\ry=y \mid \Phi_f(\rvx)\right)
    }{
        \mathbb{P}\left(\rh=h \mid \Phi_f(\rvx)\right)
    } .
\end{equation*}

It remains to simplify the likelihood term
$\mathbb{P}\left(\rh=h \mid \ry=y, \Phi_f(\rvx)\right)$.
Since the human is calibrated with respect to $\Phi_h$, the human prediction depends on
$\rvx$ through $\Phi_h(\rvx)$, so
\begin{equation*}
    \mathbb{P}\left(\rh=h \mid \ry=y, \Phi_f(\rvx), \Phi_h(\rvx)\right)
    =
    \mathbb{P}\left(\rh=h \mid \Phi_h(\rvx)\right).
\end{equation*}
Therefore, by the tower rule,
\begin{align*}
    \mathbb{P}\left(\rh=h \mid \ry=y, \Phi_f(\rvx)\right)
    &=
    \mathbb{E}_{\Phi_h}\left[
        \mathbb{P}\left(\rh=h \mid \Phi_h(\rvx)\right)
        \mid \ry=y, \Phi_f(\rvx)
    \right].
\end{align*}
Using the assumption
\[
    \Phi_h \ \bot \ \Phi_f(\rvx) \mid \ry,
\]
we can remove the conditioning on $\Phi_f(\rvx)$:
\begin{equation*}
    \mathbb{P}\left(\rh=h \mid \ry=y, \Phi_f(\rvx)\right)
    =
    \mathbb{E}_{\Phi_h}\left[
        \mathbb{P}\left(\rh=h \mid \Phi_h(\rvx)\right)
        \mid \ry=y
    \right].
\end{equation*}

Substituting this into Bayes rule gives
\begin{equation*}
    \mathbb{P}\left(\ry=y \mid \rh=h, \Phi_f(\rvx)\right)
    =
    \frac{
    \mathbb{E}_{\Phi_h}\left[
        \mathbb{P}\left(\rh=h \mid \Phi_h(\rvx)\right)
        \mid \ry=y
    \right]
    \cdot
    \mathbb{P}\left(\ry=y \mid \Phi_f(\rvx)\right)
    }{
    \mathbb{P}\left(\rh=h \mid \Phi_f(\rvx)\right)
    }.
\end{equation*}

Since the model is calibrated with respect to $\Phi_f$,
\[
    f_y(\rvx)=\mathbb{P}\left(\ry=y \mid \Phi_f(\rvx)\right).
\]
Thus the Bayesian combination posterior computed from $f(\rvx)$ and the human confusion probabilities equals the true conditional posterior:
\[
    p\left(\ry=y \mid \rh=h, f(\rvx)\right)
    =
    \mathbb{P}\left(\ry=y \mid \rh=h, \Phi_f(\rvx)\right).
\]
Since this holds for every $y$ and every realized value $h$, we obtain
\[
    p\left(\ry | \rh, f(\rvx) \right)
    =
    \mathbb{P}\left(\ry \mid \rh, \Phi_f(\rvx)\right)
    =
    \frac{
    \mathbb{E}_{\Phi_{h}}\left[
        \mathbb{P}(\rh | \Phi_{h}(\rvx)) 
        \mid \ry
    \right]
    \cdot 
    \mathbb{P}\left(\ry | \Phi_{f}(\rvx) \right)
    }{
    \mathbb{P}(\rh | \Phi_{f}(\rvx))
    }.
\]
Therefore the combination posterior is calibrated with respect to $\Phi_f$ in the sense of Definition \ref{def:comp_cal}.
\end{proof}

\section{Proof of Proposition \ref{prop_post_miscal_f_bayes}}\label{sec:prop_post_miscal_f_bayes}
\begin{proof} The human's prediction is drawn from $\mathbb{P}(\rh \mid \Phi_h(\rvx))$ independently of the label, i.e.~$\rh \perp \ry \mid \rvx$, and the true posterior is $\mathbb{P}(\ry \mid \rh, \rvx) = \mathbb{P}(\ry \mid \rvx) = \rvf(\rvx)$.  The Bayesian combination computes
\begin{equation*}
    p(\ry = y \mid \rh = h, \rvf(\rvx)) \ = \ \frac{c_{y,h}\, f_{y}(\rvx)}{\sum_{j \in \mathcal{Y}} c_{j,h}\, f_{j}(\rvx)}, \qquad c_{y,h} = \mathbb{P}(\rh = h \mid \ry = y).
\end{equation*}
Fix $\vx$ in the set with full support over $\mathcal{Y}$.  Since $f_{y}(\vx) > 0$ for every $y$, the two posteriors agree at $(h, \vx)$ if and only if $c_{y,h} = \sum_{j} c_{j,h} f_{j}(\vx)$ for every $y$.  Moreover, at such $\vx$,  $\mathbb{P}(\rh = h \mid \rvx = \vx) = \mathbb{P}(\ry = h \mid \Phi_h(\vx)) > 0$ whenever the cell $\Phi_h(\vx)$ meets the full-support set.  Hence, if the two posteriors agreed, $\mathbb{P}(\rh = h \mid \ry = y)$ would be constant in $y$ for every $h$, i.e.~$\rh \perp \ry$, contradicting the supposition that the human is informative.  Therefore $p(\ry \mid \rh, \rvf(\rvx)) \neq \mathbb{P}(\ry \mid \rvx)$.
\end{proof}

\section{Proof of Theorem \ref{thm:cal_rejector}}\label{proof_cal_rejector}
\begin{proof}
The conditional $0$--$1$ risk at $\rvx$ in Equation \ref{eq:cal_01_risk} is
\begin{equation*}
    (1 - r(\rvx))\!\left(1 - \max_{y} f_{y}(\rvx)\right) \ + \ r(\rvx)\!\left(1 - \mathbb{P}(\rh = \ry \mid \rvx)\right),
\end{equation*}
which is minimized pointwise by $r^{*}(\rvx) = \1[\mathbb{P}(\rh = \ry \mid \rvx) > \max_{y} f_{y}(\rvx)] = \1_{\mathfrak{D}^{*}}(\rvx)$.
Subtracting the pointwise optimum yields the excess risk
\begin{equation}\label{eq:excess}
    \mathcal{R}^{0\text{-}1}_{f,h}(r) - \mathcal{R}^{0\text{-}1}_{f,h}(r^{*})
    \ = \ \mathbb{E}_{\rvx}\!\left[\, \bigl| \mathbb{P}(\rh = \ry \mid \rvx) - \max_{y} f_{y}(\rvx) \bigr|
    \cdot \1\{ r(\rvx) \neq \1_{\mathfrak{D}^{*}}(\rvx) \} \,\right] \ \ge \ 0 ,
\end{equation}
since whenever $r(\rvx)$ takes the suboptimal value it incurs the margin
$|\mathbb{P}(\rh = \ry \mid \rvx) - \max_{y} f_{y}(\rvx)|$.  As ties have probability zero,
Equation \ref{eq:excess} vanishes if and only if $r(\rvx) = \1_{\mathfrak{D}^{*}}(\rvx)$ almost everywhere.

Fix a $\tilde{\Phi}_{r}$-cell $A$.  By assumption $\Phi_{r}$ refines the model-output coordinate, so $\rvf$ is $\tilde{\Phi}_{r}$-measurable and $\max_{y} f_{y}$ is constant on $A$; write $c_{A} \triangleq \max_{y} f_{y}(\rvx)$ for $\rvx \in A$.  By the assumed calibration of $\rrho$ w.r.t.~$\tilde{\Phi}_{r}$, the score is likewise constant on $A$, with value
\begin{equation*}
\rrho_{A} \ = \ \mathbb{P}(\rh = \ry \mid \tilde{\Phi}_{r}(\rvx) = A) \ = \ \mathbb{E}[\mathbb{P}(\rh = \ry \mid \rvx) \mid \rvx \in A].
\end{equation*}
Since both $\rrho$ and $\max_{y} f_{y}$ are constant on $A$, the rejector $r(\rvx;\rrho) = \1[\rrho_{A} > c_{A}]$ is constant on $A$ as well.

\textbf{($\Leftarrow$)} Assume $\sigma(\1_{\mathfrak{D}^{*}}) \subseteq \sigma(\tilde{\Phi}_{r})$. Then each $\tilde{\Phi}_{r}$-cell $A$ lies entirely in $\mathfrak{D}^{*}$ or in $(\mathfrak{D}^{*})^{c}$. If $A \subseteq \mathfrak{D}^{*}$,
then $\mathbb{P}(\rh = \ry \mid \rvx) > c_{A}$ on $A$ (with $c_{A}$ constant on $A$), so
$\rrho_{A} = \mathbb{E}[\mathbb{P}(\rh = \ry \mid \rvx) \mid A] > c_{A}$ and
$r(\rvx;\rrho) = 1 = \1_{\mathfrak{D}^{*}}(\rvx)$ on $A$. If $A \subseteq (\mathfrak{D}^{*})^{c}$,
then $\mathbb{P}(\rh = \ry \mid \rvx) < c_{A}$ on $A$, so $\rrho_{A} < c_{A}$ and
$r(\rvx;\rrho) = 0 = \1_{\mathfrak{D}^{*}}(\rvx)$ on $A$. Hence $r(\rvx;\rrho) = \1_{\mathfrak{D}^{*}}$, and by Equation \ref{eq:excess} the excess risk is zero.

\textbf{($\Rightarrow$)} Suppose $r(\rvx;\rrho)$ is risk-minimizing. The score $\rrho(\rvx,\rvf(\rvx))$ is $\tilde{\Phi}_{r}$-measurable by the assumed calibration, and $\rvf(\rvx)$ is $\tilde{\Phi}_{r}$-measurable by the refinement hypothesis; hence
$r(\rvx;\rrho) = \1[\rrho > \max_{y} f_{y}]$ is $\tilde{\Phi}_{r}$-measurable. By Equation \ref{eq:excess} and the no-tie assumption,
$r(\rvx;\rrho) = \1_{\mathfrak{D}^{*}}(\rvx)$ almost everywhere. Therefore $\1_{\mathfrak{D}^{*}}$
agrees with a $\tilde{\Phi}_{r}$-measurable function, giving
$\sigma(\1_{\mathfrak{D}^{*}}) \subseteq \sigma(\tilde{\Phi}_{r})$.
\end{proof}

\section{Experimental Details for Bayesian Combination Experiments}\label{bayes_comb_exper_details}

All experiments were run on a SLURM cluster with nodes of 8 x NVIDIA RTX A5000 (24 GB) with 64 CPU cores and 256 GB RAM per node. All jobs used a single GPU. The software we used was: Python 3.10.20, PyTorch 2.5.1+cu121, torchvision 0.20.1+cu121, NumPy 2.2.6, scikit-learn 1.7.2, pandas 2.3.3. Regarding runtime, no single experimental training run or temperature sweep required more than 20 minutes of computation.

\subsection{ImageNet-16H} 
We use ImageNet-16H from the public OSF release of \citet{kerrigan2021}, which contains human classifications and precomputed CNN softmax outputs on a 16-class subset of ImageNet: \url{https://osf.io/2ntrf/overview}. We focus on the VGG-19 predictions at image noise level $80$ and evaluate the performance of the resulting human-AI combination as the classifier's standalone calibration changed.  For each image $\vx$, the dataset provides approximately six human class annotations. We aggregate these annotations into a count vector over the $K=16$ classes and interpret the normalized count vector as the empirical human response distribution for that image. To match the theoretical setting in which we observe a sampled human prediction rather than the human's full internal distribution, we draw a single hard human label for each image.  

We split the examples into a calibration subset and a held-out evaluation subset. For each random run, we use a stratified $70/30$ split. The calibration subset is used only to estimate the human confusion matrix $\mC$, while the held-out subset is used to evaluate the standalone classifier and the combined predictor. In all ImageNet-16H experiments reported here, we use all calibration examples to estimate $\mC$. We repeat the procedure over five random seeds, where each run re-samples the calibration/evaluation split and independently samples the hard human labels on both subsets.  

To vary the calibration of the classifier without changing its learned representation, we apply fixed temperature scaling to the VGG-19 softmax probabilities. Smaller temperatures sharpen the probability vector, whereas larger temperatures flatten it. We sweep prespecified grids of temperatures and evaluate both the standalone classifier and the confusion-matrix combined predictor using the same held-out examples.

We report three quantities: mass-binned expected calibration error, negative log-likelihood, and accuracy. For ECE, we use $15$ equal-mass confidence bins. The main plots place standalone model ECE on the horizontal axis and a combined-predictor metric on the vertical axis. Each point corresponds to one fixed temperature $T$, and the plotted value is averaged across random runs.

\subsection{HAM10000}
We use the HAM10000 skin-lesion classification dataset to construct a controlled setting in which the simulated expert prediction can depend on different information sources. The training set is used to train the image classifier and expert predictors, the calibration set is used to estimate expert confusion matrices, and the test set is used for final evaluation. Each example contains a dermoscopic image and patient- or lesion-level meta-data. We identify the image with the model-observed feature $\rvx$ and treat the meta-data as privileged information $\rvz$ that is unavailable to the base classifier but may be available to an expert. We simulate the expert prediction using trained auxiliary predictors. This allows us to compare several feature-access regimes while keeping the same confusion-matrix update framework.

The model is a temperature-scaled ResNet-18 trained on $224 \times 224$ dermoscopic images for 8 epochs. We compare four expert settings. The first, labeled \emph{Image}, uses the same ResNet-18 output as the main classifier, representing an identical-image, duplicated-evidence expert. The second, labeled \emph{Weaker Image}, is a separate weaker ResNet-18 trained on $128 \times 128$ images for 4 epochs, representing a shared-feature setting with a different image predictor. The third, labeled \emph{Image and Meta-Data}, is a multimodal expert consisting of a ResNet-18 image encoder and a meta-data MLP, representing a partial-overlap setting because it uses both the image and patient meta-data. The fourth, labeled \emph{Meta-Data Only}, is a logistic regression expert trained on age, sex, and lesion localization. In contrast to the previous three experts, it depends only on $\rvz$ and not on the image feature $\rvx$, making it the most independent expert relative to the image classifier.

For each expert, we treat its softmax output as the expert's internal predictive distribution and sample hard labels from it. Thus, as in the ImageNet-16H experiment, the combination method observes only a sampled categorical expert prediction rather than the full expert distribution. On the calibration split, we sample four hard expert labels per example to estimate $\mC$ more stably. On the test split, we sample one hard expert label per example and combine it with the temperature-scaled image classifier.  As before, we vary only the temperature of the base image classifier and keep the expert distributions fixed. This isolates how the calibration of the model affects the resulting combined predictor under different expert information structures. In all plots, the horizontal axis is the mass-binned ECE of the standalone image classifier, and the vertical axis is the metric of the combined predictor. For all experiments, we average over $5$ random sampling runs.  The optimizer was AdamW, the learning rate was $10^{-4}$, augmentations were random horizontal flips and rotations, and the split sizes were $6,958$ train / $1,502$ validation / $1,498$ test. 

\subsection{CIFAR-10H}
\label{app:cifar10h_validation}

\begin{figure}[ht]
    \centering
    \begin{subfigure}[b]{0.329\textwidth}
        \centering
        \includegraphics[width=\linewidth]{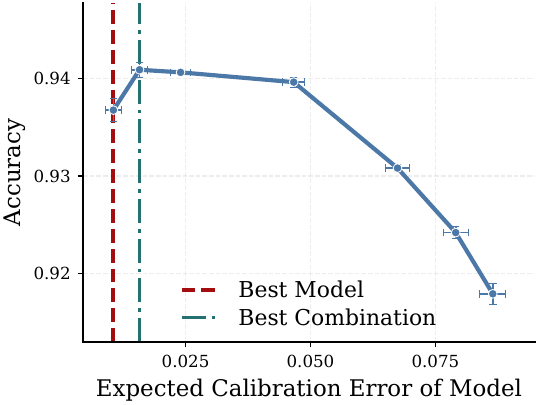}
        \caption{Accuracy}
        \label{fig:cifar10h_acc}
    \end{subfigure}
    \begin{subfigure}[b]{0.329\textwidth}
        \centering
        \includegraphics[width=\linewidth]{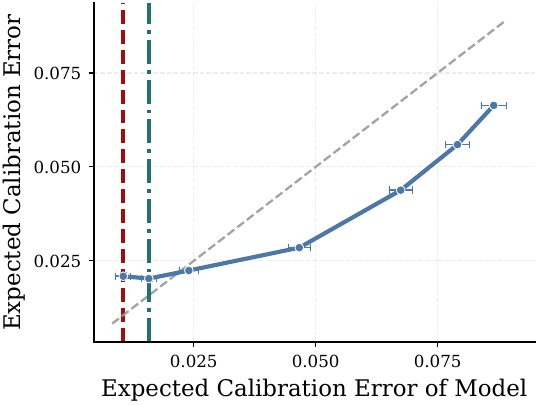}
        \caption{Calibration Error}
        \label{fig:cifar10h_ece}
    \end{subfigure}
    \begin{subfigure}[b]{0.329\textwidth}
        \centering
        \includegraphics[width=\linewidth]{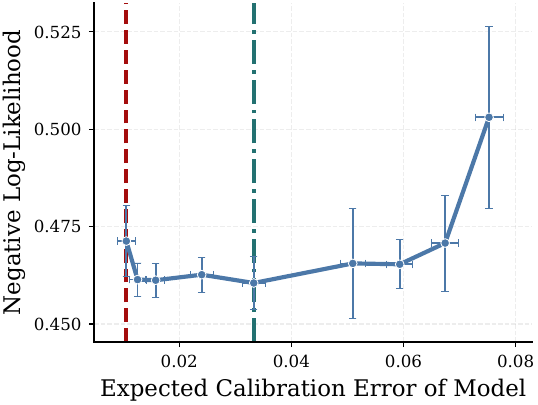}
        \caption{Neg. Log-Likelihood}
        \label{fig:cifar10h_nll}
    \end{subfigure}
    \caption{\textit{Additional Validation on CIFAR-10H.}
    We vary the calibration of a ResNet classifier by temperature scaling and combine each temperature-scaled classifier with a sampled human label using the confusion-matrix update. The human label is sampled from the CIFAR-10H empirical human distribution after flattening with human temperature $T_h=2.0$. The horizontal axis reports the standalone model's mass-binned ECE, while the vertical axes report the combined predictor's accuracy, ECE, and NLL.  Error bars report one standard error.}
    \label{fig:cifar10h_appendix}
\end{figure}

We additionally evaluate the confusion-matrix combination method on CIFAR-10H~\citep{peterson2019human}, a human uncertainty dataset built on the CIFAR-10 test set. CIFAR-10H provides human label distributions for each image, which allows us to sample hard human predictions while preserving real human perceptual uncertainty.  We use ResNet-18 as the base image classifier. For each image, we sample one hard human label from the empirical CIFAR-10H human label distribution. To study a weaker-human setting, we flatten the empirical human distribution using a human temperature of $T_h=2.0$ before sampling. We do not add extra random label corruption. Thus, the sampled human labels remain tied to the real per-image human uncertainty in CIFAR-10H, but with reduced confidence.

For each run, we split the data into a calibration subset and a held-out evaluation subset using a stratified $70/30$ split. The calibration subset is used to estimate the human confusion matrix $\mC$, and the held-out subset is used to evaluate both the standalone temperature-scaled classifier and the combined predictor. We use $5{,}000$ calibration examples to estimate $\mC$ and average results over five random runs. In each run, the train-test split and the sampled human labels are regenerated. We vary only the classifier temperature while keeping the human sampling procedure fixed. We report combined accuracy, combined ECE, and combined NLL as functions of the standalone model's mass-binned ECE. The results in Figure~\ref{fig:cifar10h_appendix} provide an additional real-human-label validation that the best standalone model calibration need not correspond to the best combined predictor.


\section{Robustness to Noisy and Biased Human Predictions}
\label{app:noisy_biased_human}

Assumption~\ref{fund_assump} requires the human to be calibrated w.r.t.~$\Phi_h$, yet in practice human predictions may be noisy, systematically biased, or otherwise misleading.  Our theoretical results are stated for the good-case scenario in which calibration holds, and we do not propose a remedy for its violation.  Nevertheless, it is natural to ask whether the qualitative prediction of Section~\ref{sec:bayes_comb_results} holds when the human is corrupted.  This experiment addresses that question.  We use the HAM10000 setting of Appendix~\ref{bayes_comb_exper_details} and Figures~\ref{fig:ham10k_ece} and~\ref{fig:ham10k_nll} without modification, except that the sampled expert label $\hat{\rh}$ is corrupted before it is passed to the confusion-matrix combination.  The corruption is applied identically on the calibration split (where $\mC$ is estimated) and on the test split (where the posterior is evaluated), so that the combination method sees a corrupted human throughout, exactly as it would in deployment.  The four simulated experts (\emph{Image}, \emph{Weaker Image}, \emph{Image and Meta-Data}, \emph{Meta-Data Only}), the temperature-scaled ResNet-18 model, the temperature grid, and the five random seeds are as described in Appendix~\ref{bayes_comb_exper_details}.  We consider two types of corruption:
\begin{itemize}[leftmargin=15pt, itemsep=2pt]
    \item \textbf{Noise}:  With probability $\epsilon \in \{0.1, 0.2, 0.3\}$, the sample $\hat{\rh}$ is replaced by a label drawn uniformly at random from $\mathcal{Y} \setminus \{\ry\}$, i.e.~a uniformly random \emph{incorrect} class.  Because the replacement is never correct, this is label-dependent noise that strictly reduces the human's accuracy.
    \item \textbf{Bias}:  With probability $b \in \{0.2, 0.4\}$, the sampled label $\hat{\rh}$ is replaced by a fixed class: either the majority class or the minority class of the training set.  
    This models a human with a systematic tendency to over-predict a particular diagnosis.
\end{itemize}
This yields seven corruption configurations in total: three noise levels, and two bias strengths for each of the two target classes.  In every configuration the corrupted human is, by construction, no longer calibrated w.r.t.~$\Phi_h$, so Assumption~\ref{fund_assump} is violated for the human. 

\begin{table}[t]
    \centering
    \caption{\emph{Robustness to a noisy or biased human.}  For each simulated HAM10000 expert and each of seven corruption configurations, we check whether the best-calibrated classifier also yields the best posterior predictor.  Entries give the number of configurations (out of 7) for which it does.  We expect $7/7$ for the independent-feature \emph{Meta-Data Only} expert and $0/7$ for all other experts.}
    \label{tab:noisy_biased_human}
    \vspace{0.5em}
    \begin{tabular}{llcc}
        \toprule
        Simulated Expert & Metric & Best Model $=$ Best Posterior & Coincident Configurations\\
        \midrule
        \emph{Image}               & ECE & $0/7$ & None \\
                                                     & NLL & $0/7$ & None \\
        \midrule
        \emph{Weaker Image}        & ECE & $0/7$ & None \\
                                                     & NLL & $1/7$ & Noisy with $\epsilon = 0.3$ \\
        \midrule
        \emph{Image and Meta-Data} & ECE & $0/7$ & None \\
                                                     & NLL & $0/7$ & None \\
        \midrule
        \emph{Meta-Data Only}      & ECE & $7/7$ & All \\
                                                     & NLL & $7/7$ & All \\
        \bottomrule
    \end{tabular}
\end{table}

Table~\ref{tab:noisy_biased_human} reports, for each expert and metric, the number of the seven configurations for which the best team is formed by the best-calibrated model.  The qualitative prediction holds in $55$ of the $56$ expert--metric--configuration combinations.  For the \emph{Meta-Data Only} expert, the best-calibrated classifier yields the best posterior in both ECE and NLL under every one of the seven corruptions, as the independent-feature analysis predicts.  For the three experts whose features overlap with the image classifier's, the best-calibrated classifier never yields the best posterior ECE, and yields the best posterior NLL in a single case: the \emph{Weaker Image} expert under the strongest noise level, $\epsilon = 0.3$.  This one deviation is plausibly explained by the corruption itself---at $30\%$ noise the weaker image expert carries little usable information, possibly violating the assumption that the human is an informative predictor.

\section{LLM Use}
Claude Opus 4.6 was used to check the proofs for correctness when writing the initial draft.  Later, Claude Fable 5.1 was used to re-check proofs for correctness and to draft revised proofs.  OpenAI Codex was used to help implement the experiments reported in Figures 3 and 5.  All initial writing was done without LLMs.  Claude Opus 4.6 proofread the draft before submission, and Claude Fable 5.1 provided proofreading before the camera-ready deadline.

\end{document}